\pdfoutput=1
\documentclass[11pt]{article} \usepackage{times}
\usepackage[letterpaper, left=1in, right=1in, top=1in, bottom=1in]{geometry}

\usepackage[utf8]{inputenc} \usepackage[T1]{fontenc}    \usepackage{url}            \usepackage{booktabs}       \usepackage{amsfonts}       \usepackage{nicefrac}       \usepackage{microtype}      \usepackage{mkolar_definitions}
\usepackage{xspace}
\usepackage{amsmath}
\usepackage{algorithm,algorithmic}
\usepackage{color}
\usepackage{enumitem}
\usepackage{comment}
\usepackage{bm}
\usepackage{graphicx}

\usepackage{apptools}
\usepackage[page, header]{appendix}
\usepackage{titletoc}

\usepackage[scaled=.9]{helvet}

\usepackage{url}
\usepackage{authblk}
\usepackage{caption}

\usepackage[dvipsnames]{xcolor}
\usepackage[colorlinks=true, linkcolor=blue, citecolor=blue]{hyperref}

\usepackage{tikz}
\usetikzlibrary{decorations.pathreplacing}
\usepackage{wrapfig}

\newcommand{\magic}{\textsc{Magic}\xspace}

\newcommand{\bias}{\textsc{b}\xspace}
\newcommand{\biast}{\textsc{bias}\xspace}
\newcommand{\dev}{\textsc{dev}\xspace}
\newcommand{\var}{\textsc{var}\xspace}
\newcommand{\const}{\textsc{const}\xspace}

\newcommand{\conf}{\textsc{cnf}\xspace}

\newcommand{\pilog}{\pi_{\mathrm{L}}}
\newcommand{\pitarget}{\pi_{\mathrm{T}}}

\newcommand{\vdr}{\hat{V}_{\mathrm{DR}}}
\newcommand{\vdm}{\hat{V}_{\mathrm{DM}}}
\newcommand{\vmax}{V_{\mathrm{max}}}
\newcommand{\pmax}{p_{\mathrm{max}}}
\newcommand{\pmin}{p_{\mathrm{min}}}

\newcommand{\monoconst}{\kappa}

\newcommand{\lepski}{\textsc{Slope}\xspace}
 \newcommand{\version}{arxiv}

\usepackage{natbib}
\bibpunct{(}{)}{;}{a}{,}{,}

\newenvironment{packed_enum}{
  \begin{enumerate}
    \setlength{\itemsep}{1pt}
    \setlength{\parskip}{-1pt}
    \setlength{\parsep}{0pt}
}{\end{enumerate}}
\newenvironment{packed_itemize}{
  \begin{itemize}
    \setlength{\itemsep}{1pt}
    \setlength{\parskip}{-1pt}
    \setlength{\parsep}{0pt}
}{\end{itemize}}

\newcount\Comments  \Comments=1 \definecolor{darkgreen}{rgb}{0,0.5,0}
\definecolor{darkred}{rgb}{0.7,0,0}
\definecolor{teal}{rgb}{0.3,0.8,0.8}
\definecolor{orange}{rgb}{1.0,0.5,0.0}
\definecolor{purple}{rgb}{0.8,0.0,0.8}
\newcommand{\kibitz}[2]{\ifnum\Comments=1{\textcolor{#1}{\textsf{\footnotesize #2}}}\fi}

\usepackage{authblk}

\title{Adaptive Estimator Selection for Off-Policy Evaluation}
\date{}
\author[1]{Yi Su\thanks{ys756@cornell.edu}}
\author[2]{Pavithra Srinath\thanks{pavithraks@gmail.com}}
\author[2]{Akshay Krishnamurthy\thanks{akshaykr@microsoft.com}}
\affil[1]{Cornell University, Ithaca, NY}
\affil[2]{Microsoft Research, New York, NY}

\begin{document}

\maketitle

\vspace{-1cm}
\begin{abstract}
We develop a generic data-driven method for estimator selection in
off-policy policy evaluation settings. We establish a strong
performance guarantee for the method, showing that it is competitive
with the oracle estimator, up to a constant factor. Via in-depth case
studies in contextual bandits and reinforcement learning, we
demonstrate the generality and applicability of the method. We also
perform comprehensive experiments, demonstrating the empirical
efficacy of our approach and comparing with related approaches. In
both case studies, our method compares favorably with existing
methods.
 \end{abstract}

\section{Introduction}
\label{sec:intro}

In practical scenarios where safety, reliability, or performance is a
concern, it is typically infeasible to directly deploy a reinforcement
learning (RL) algorithm, as it may compromise these desiderata. This
motivates research on \emph{off-policy evaluation}, where we use data
collected by a (presumably safe) logging policy to estimate the
performance of a given target policy, without ever deploying it. These
methods help determine if a policy is suitable for deployment at
minimal cost and, in addition, serve as the statistical foundations of
sample-efficient policy optimization algorithms. In light of the
fundamental role off-policy evaluation plays in RL, it has been the
subject of intense research over the last several
decades~\citep{horvitz1952generalization,dudik2014doubly,swaminathan2017off,kallus2018policy,sutton1988learning,bradtke1996linear,precup2000eligibility,jiang2015doubly,thomas2016data,farajtabar2018more,liu2018breaking,voloshin2019empirical}.

As many off-policy estimators have been developed, practitioners face
a new challenge of choosing the best estimator for their
application. This selection problem is critical to high quality
estimation as has been demonstrated in recent empirical
studies~\cite{voloshin2019empirical}. However, data-driven estimator
selection in these settings is fundamentally different from
hyperparameter optimization or model selection for supervised
learning. In particular, cross validation or bound minimization
approaches fail because there is no unbiased and low variance approach
to compare estimators. As such, the current best practice for
estimator selection is to leverage domain expertise or follow
guidelines from the literature~\citep{voloshin2019empirical}.

Domain knowledge can suggest a particular form of estimator, but a
second selection problem arises, as many estimators themselves have
hyperparameters that must be tuned. In most cases, these
hyperparameters modulate a bias-variance tradeoff, where at one
extreme the estimator is unbiased but has high variance, and at the
other extreme the estimator has low variance but potentially high
bias. Hyperparameter selection is critical to performance, but
high-level prescriptive guidelines are less informative for these
low-level selection problems. We seek a data-driven approach.

In this paper, we study the estimator-selection question for
off-policy evaluation. We provide a general technique, that we call
\lepski, that applies to a broad family of estimators, across several
distinct problem settings. On the theoretical side, we prove that the
selection procedure is competitive with oracle tuning, establishing an
\emph{oracle inequality}. To demonstrate the generality of our
approach, we study two applications in detail: (1) bandwidth selection
in contextual bandits with continuous actions, and (2) horizon
selection for ``partial importance weighting estimators'' in RL. In
both examples, we prove that our theoretical results apply, and we
provide a comprehensive empirical evaluation. In the contextual
bandits application, our selection procedure is competitive with the
skyline oracle tuning (which is unimplementable in practice) and
outperforms any fixed parameter in aggregate across experimental
conditions.  In the RL application, our approach substantially
outperforms standard baselines including
\magic~\citep{thomas2016data}, the only comparable estimator selection
method.

\ifthenelse{\equal{\version}{arxiv}}{
\begin{wrapfigure}{R}{0.5\textwidth}
\vspace{-0.8cm}
\begin{center}
\includegraphics[width=0.47\textwidth]{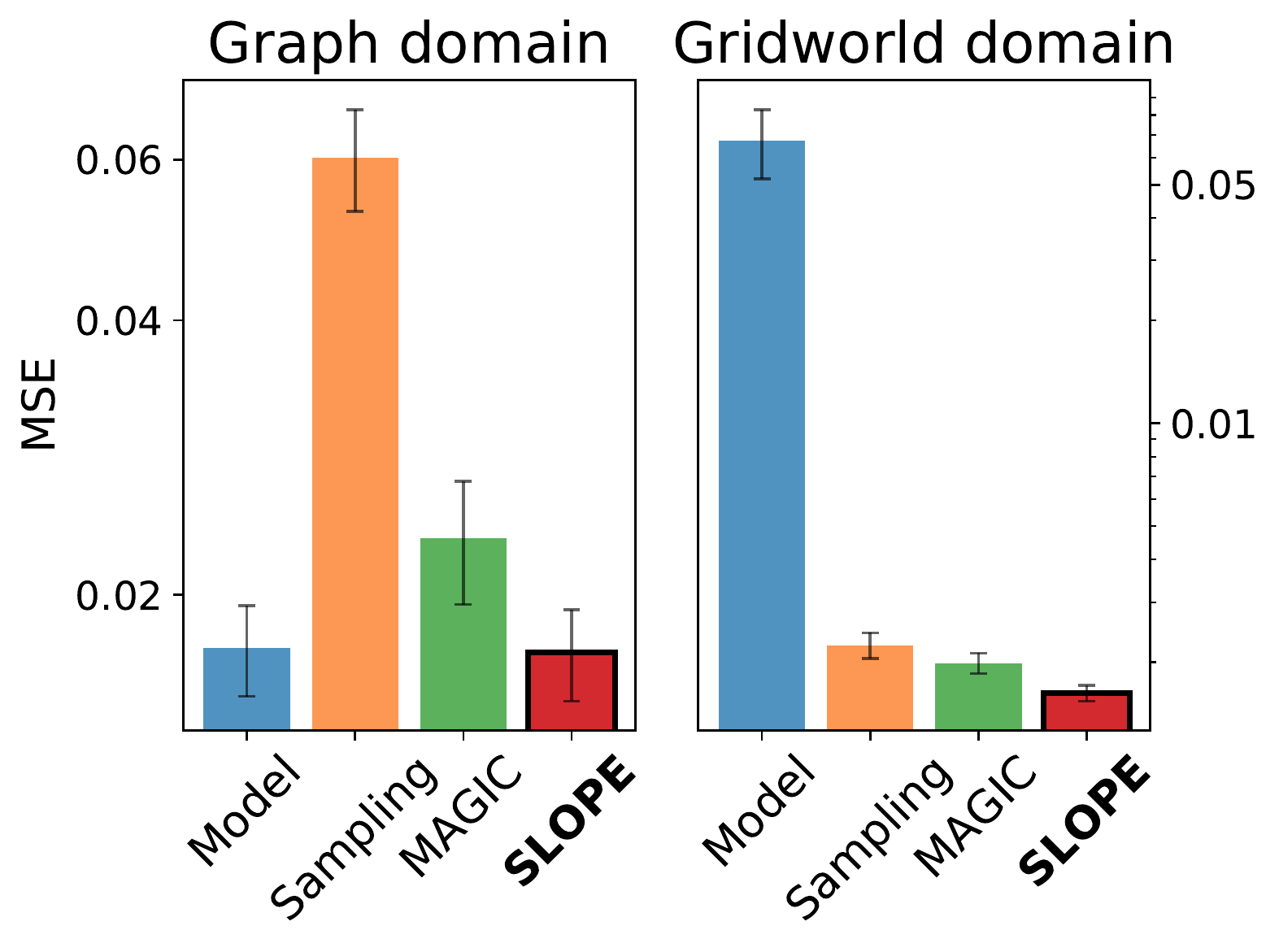}
\end{center}
\vspace{-0.6cm}
\caption{Representative experiments with \lepski. \lepski is consistently one of the best
  approaches, regardless of whether model-based or importance-sampling
  based estimators are better.}
\label{fig:intro_exp}
\vspace{-0.3cm}
\end{wrapfigure}
}{
\begin{figure}
\begin{center}
\includegraphics[width=0.8\columnwidth]{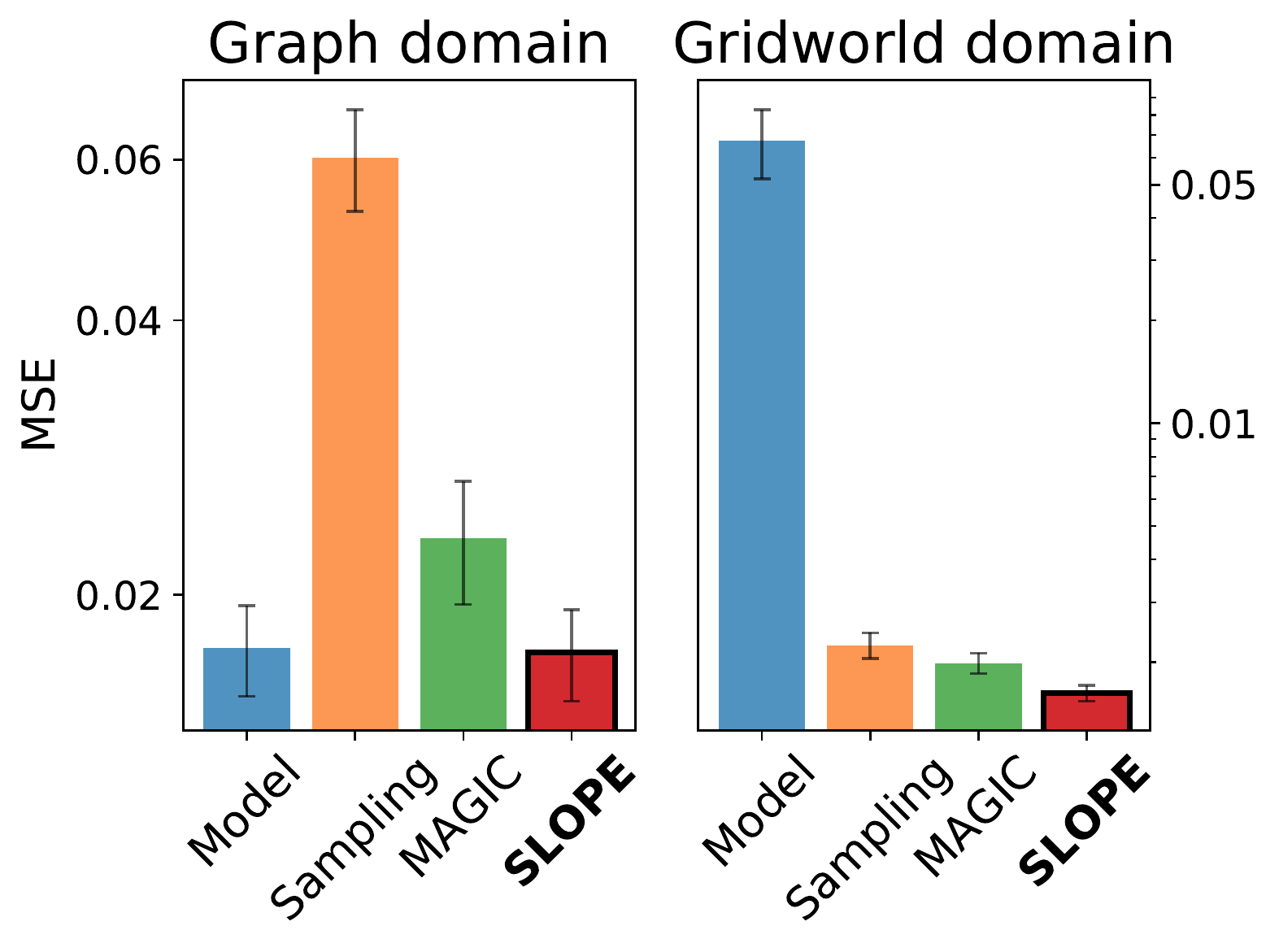}
\end{center}
\vspace{-0.5cm}
\caption{Representative experiments with \lepski. \lepski is consistently one of the best
  approaches, regardless of whether model-based or importance-sampling
  based estimators are better.}
\label{fig:intro_exp}
\vspace{-0.4cm}
\end{figure}
}

A representative experimental result for the RL setting is displayed
in~\pref{fig:intro_exp}. Here we consider two different domains
from~\citet{voloshin2019empirical} and compare our new estimator,
\lepski, with well-known baselines. Our method selects a false horizon
$\eta$, uses an unbiased importance sampling approach up to horizon
$\eta$, and then prematurely terminates the episode with a value
estimate from a parametric estimator (in this case trained via Fitted
Q iteration). Model selection focuses on choosing the false horizon
$\eta$, which yields parametric and trajectory-wise importance
sampling estimators as special cases (``Model'' and ``sampling'' in
the figure). Our experiments show that regardless of which of these
approaches dominates, \lepski is competitive with the
best approach. Moreover, it outperforms \magic, the only other tuning
procedure for this setting.~\pref{sec:rl} contains more details and
experiments.

At a technical level, the fundamental challenge with estimator
selection is that there is no unbiased and low-variance approach for
comparing parameter choices. This precludes the use of cross
validation and related approaches, as estimating the error of a method
is itself an off-policy evaluation problem!
Instead, adapting ideas from nonparametric
statistics~\citep{lepski1997optimal,mathe2006lepskii}, our selection
procedure circumvents this error estimation problem by only using
variance estimates, which are easy to obtain. At a high level, we use
confidence bands for each estimator around their (biased) expectation
to find one that approximately balances bias and variance. This
balancing corresponds to the oracle choice, and so we obtain our
performance guarantee.

\paragraph{Related work.}
As mentioned, off-policy evaluation is a vibrant research area with
contributions from machine learning, econometrics, and statistics
communities.
Two settings of particular interest are contextual bandits and general
RL. For the former, recent and classical references
include~\cite{horvitz1952generalization,dudik2014doubly,hirano2003efficient,farajtabar2018more,su2019doubly}. For
the latter, please refer to~\citet{voloshin2019empirical}.

Parameter tuning is quite important for many off-policy evaluation
methods.~\citet{munos2016safe} observe that methods like
\textsc{Retrace} are fairly sensitive to the hyperparameter. Similarly
conclusions can be drawn from the experiments of~\citet{su2019cab} in
the contextual bandits context. Yet,
when tuning is required, most works resort to heuristics.  For
example, in~\citet{kallus2018policy}, a bandwidth hyperparameter is
selected by performing an auxiliary nonparametric estimation task,
while in~\citet{liu2018breaking}, it is selected as the median of the
distances between points. In both cases, no theoretical guarantees are
provided for such methods.

Indeed, despite the prevalence of hyperparameters in these methods, we
are only aware of two methods for estimator selection: the \magic
estimator~\citep{thomas2016data}, and the bound minimization approach
studied by~\citet{su2019doubly} (see also~\citet{wang2017optimal}).
Both approaches use MSE surrogates for estimator selection, where
\magic under-estimates the MSE and the latter uses an over-estimate.
The guarantees for both methods (asymptotic consistency, competitive
with unbiased approaches) are much weaker than our oracle inequality,
and \lepski substantially outperforms \magic in experiments.

Our approach is based on Lepski's principle for bandwidth selection in
nonparametric
statistics~\citep{lepskii1992asymptotically,lepskii1991problem,lepskii1993asymptotically,lepski1997optimal}.
In this seminal work, Lepski studied nonparametric estimation problems
and developed a data-dependent bandwidth selection procedure that
achieves optimal adaptive guarantees, in settings where procedures
like cross validation do not apply (e.g., estimating a regression
function at a single given point). Since its introduction, Lepski's
methodology has been applied to other statistics
problems~\cite{birge2001alternative,mathe2006lepskii,goldenshluger2011bandwidth,kpotufe2013adaptivity,page2018goldenshluger},
but its use in machine learning has been limited. To our knowledge,
Lepski's principle has not been used for off-policy evaluation, which
is our focus.

\section{Setup}
\label{sec:setup}

We formulate the estimator selection problem generically, where there is an
abstract space $\Zcal$ and a distribution $\Dcal$ over $\Zcal$. We
would like to estimate some parameter $\theta^\star \defeq
\theta(\Dcal) \in \RR$, where $\theta$ is some known real-valued
functional, given access to $z_1,\ldots,z_n \iidsim \Dcal$. Let
$\hat{\Dcal}$ denote the empirical measure, that is the uniform
measure over the points $z_{1:n}$. 

To estimate $\theta^\star$ we use a finite set of $M$ estimators
$\{\theta_i\}_{i=1}^M$, where each $\theta_i: \Delta(\Zcal) \to
\RR$. Given the dataset, we form the estimates $\hat{\theta}_i \defeq
\theta_i(\hat{\Dcal})$. Ideally, we would choose the index that
minimizes the absolute error with $\theta^\star$, that is $\argmin_{i
  \in [M]} \abr{\hat{\theta}_i - \theta^\star}$. Of course this
\emph{oracle} index depends on the unknown parameter $\theta^\star$,
so it cannot be computed from the data.  Instead we seek a data-driven
approach for selecting an index $\hat{i}$ that approximately minimizes
the error.

To fix ideas, in the RL context, we may think of $\theta$ as the value
of a \emph{target policy} $\pitarget$ and $z_{1:n}$ as $n$
trajectories collected by some \emph{logging policy} $\pilog$. The
estimators $\{\theta_i\}_{i=1}^M$ may be partial importance weighting
estimators~\cite{thomas2016data}, that account for policy mismatch on
trajectory prefixes of different length. 
These estimators modulate a bias variance tradeoff: importance
weighting short prefixes will have high bias but low variance, while
importance weighting the entire trajectory will be unbiased but have
high variance. We will develop this example in detail
in~\pref{sec:rl}.

For performance guarantees, we decompose the absolute error into two
terms: the bias and the deviation. For this decomposition, define
$\bar{\theta}_i \defeq \EE[\theta_i(\hat{D})]$ where the expectation
is over the random samples $z_{1:n}$. Then we have
\begin{align*}
\abr{\hat{\theta}_i - \theta^\star} \leq \abr{\bar{\theta}_i - \theta^\star} + \abr{\hat{\theta}_i - \bar{\theta}_i} \eqdef \biast(i) + \dev(i).
\end{align*}
As $\dev(i)$ involves statistical fluctuations only, it is amenable to
concentration arguments, so we will assume access to a high
probability upper bound. Namely, our procedure uses a confidence
function $\conf$ that satisfies $\dev(i) \leq \conf(i)$ for all $i \in
[M]$ with high probability. On the other hand, estimating the bias is
much more challenging, so we do not assume that the estimator has
access to $\biast(i)$ or any sharp upper bound. Our goal is to select
an index $\hat{i}$ achieving an \emph{oracle inequality} of the form
\begin{align}
\label{eq:guarantee_form}
\abr{\hat{\theta}_{\hat{i}} - \theta^\star}
\leq \const \times \min_{i}\cbr{\bias(i) + \conf(i)},
\end{align}
that holds with high probability where $\const$ is a universal
constant and $\bias(i)$ is a sharp upper bound on
$\biast(i)$.\footnote{Some assumptions prevent us from setting
  $\bias(i) = \biast(i)$.}  
This guarantee certifies that the
selected estimator is competitive with the error bound for the best
estimator under consideration.

We remark that the above guarantee is qualitatively similar, but
weaker than the ideal guarantee of competing with the actual error of
the best estimator (as opposed to the error bound). In theory, this
difference is negligible as the two guarantees typically yield the
same statistical conclusions in terms of convergence
rates. Empirically we will see that ~\pref{eq:guarantee_form} does
yield strong practical performance.

\section{General Development}
\label{sec:general}

To obtain an oracle inequality of the form
in~\pref{eq:guarantee_form}, we require some benign assumptions. When
we turn to the case studies, we will verify that these assumptions
hold for our estimators.

\paragraph{Validity and Monotonicity.}
The first basic property on the bias and confidence functions is that
they are valid in the sense that they actually upper bound the
corresponding terms in the error decomposition. 
\begin{assum}[Validity]
\label{assum:concentration}
We assume
\ifthenelse{\equal{\version}{arxiv}}{}{\vspace{-3pt}}
\begin{packed_enum}
\item (Bias Validity) $\abr{\bar{\theta}_i - \theta^\star} \leq
  \bias(i)$ for all $i$.
\item (Confidence Validity) With probability at least $1-\delta$,
  $|\hat{\theta}_i - \bar{\theta}_i| \leq \conf(i)$  for all $i$.
\end{packed_enum}
\ifthenelse{\equal{\version}{arxiv}}{}{\vspace{-3pt}}
\end{assum}
Typically $\conf$ can be constructed using straightforward
concentration arguments such as Bernstein's inequality. Importantly,
$\conf$ does not have to account for the bias, so the term $\dev$ that
we must control is centered. We also note that $\conf$ need not be
deterministic, for example it can be derived from empirical Bernstein
inequalities. We emphasize again that the estimator does not have
access to $\bias(i)$.

We also require a monotonicity property on these functions.
\begin{assum}[Monotonicity]
\label{assum:monotonicity}
We assume that there exists a constant $\monoconst > 0$ such that for all
$i \in [M-1]$
\ifthenelse{\equal{\version}{arxiv}}{}{\vspace{-3pt}}
\begin{packed_enum}
\item $\bias(i) \leq \bias(i+1)$.
\item $\monoconst\cdot\conf(i) \leq \conf(i+1) \leq \conf(i)$.
\end{packed_enum}
\ifthenelse{\equal{\version}{arxiv}}{}{\vspace{-3pt}}
\end{assum}

In words, the estimators should be ordered so that the bias is
monotonically increasing and the confidence is decreasing but not too
quickly, as parameterized by the constant $\monoconst$.
This structure is quite natural when estimators navigate a
bias-variance tradeoff: when an estimator has low bias it typically
also has high variance and vice versa.
It is also straightforward to enforce a decay rate for $\conf$ by
selecting the parameter set appropriately. We will see how to do this
in our case studies.

\paragraph{The \lepski procedure.}
\lepski is an acronym for ``Selection by Lepski's principle for
Off-Policy Evaluation.'' As the name suggests, the approach is based
on Lepski's principle~\cite{lepski1997optimal} and is defined as
follows. We first define intervals
\begin{align*}
I(i) \defeq [\hat{\theta}_i - 2\conf(i), \hat{\theta}_i + 2\conf(i)],
\end{align*}
and we then use the intersection of these intervals to select an index
$\hat{i}$. Specifically, the index we select is
\begin{align*}
\hat{i} \defeq \max\cbr{i \in [M]: \cap_{j=1}^i I(j) \ne \emptyset}.
\end{align*}
In words, we select the largest index such that the intersection of
all previous intervals is non-empty. See~\pref{fig:illustration}
for an illustration.

\ifthenelse{\equal{\version}{arxiv}}{
\begin{wrapfigure}{R}{0.5\textwidth}
\begin{center}

\begin{tikzpicture}[scale=1.5]
\draw[|-|] (0,0) -- (0,2); \draw[|-|] (0,0.99) -- (0,1.01); \node at (-0.2,1) {$\hat{\theta}_1$};
\draw[|-|] (1,0.3) -- (1,1.3); \draw[|-|] (1,0.79) -- (1,0.81); \node at (0.8,0.8) {$\hat{\theta}_2$};
\draw[|-|] (2,1.0) -- (2,1.5);  \draw[|-|] (2,1.24) -- (2,1.26); \node at (1.8,1.25) {$\hat{\theta}_3$};
\draw[|-|] (3,0.4) -- (3,0.65); \draw[|-|] (3,0.515) -- (3,0.535); \node at (2.8,0.525) {$\hat{\theta}_4$};
\draw[|-|] (4,0.1) -- (4,0.225);  \draw[|-|] (4,0.1525) -- (4,0.1725); \node at (3.8,0.1625) {$\hat{\theta}_5$};
\node at (2,0) {$\hat{i} = 3$};
\draw[decoration={brace,mirror},decorate] (0.1,1) -- node[right=3pt] {$2\conf(1)$} (0.1,2);
\draw[decoration={brace,mirror},decorate] (1.1,0.3) -- node[right=3pt] {$I(2)$} (1.1,1.3);
\draw[decoration={brace,mirror},decorate] (2.1,1) -- node[right=3pt] {$I(3)$} (2.1,1.5);
\path [fill=blue,opacity=0.2] (0,1) rectangle (4,1.25);
\node[blue] at (3.5,1.125) {$\cap_{i=1}^3 I(i)$};
\end{tikzpicture}
 \end{center}
\vspace{-0.5cm}
\caption{Illustration of \lepski with $M=5$. As $\cap_{i=1}^3 I(i)$ is
  non-empty but $I(4)$ does not intersect with $I(3)$, we select
  $\hat{i}=3$.}
\vspace{-0.5cm}
\label{fig:illustration}
\end{wrapfigure}
}{
\begin{figure}
\begin{center}

\begin{tikzpicture}[scale=1.5]
\draw[|-|] (0,0) -- (0,2); \draw[|-|] (0,0.99) -- (0,1.01); \node at (-0.2,1) {$\hat{\theta}_1$};
\draw[|-|] (1,0.3) -- (1,1.3); \draw[|-|] (1,0.79) -- (1,0.81); \node at (0.8,0.8) {$\hat{\theta}_2$};
\draw[|-|] (2,1.0) -- (2,1.5);  \draw[|-|] (2,1.24) -- (2,1.26); \node at (1.8,1.25) {$\hat{\theta}_3$};
\draw[|-|] (3,0.4) -- (3,0.65); \draw[|-|] (3,0.515) -- (3,0.535); \node at (2.8,0.525) {$\hat{\theta}_4$};
\draw[|-|] (4,0.1) -- (4,0.225);  \draw[|-|] (4,0.1525) -- (4,0.1725); \node at (3.8,0.1625) {$\hat{\theta}_5$};
\node at (2,0) {$\hat{i} = 3$};
\draw[decoration={brace,mirror},decorate] (0.1,1) -- node[right=3pt] {$2\conf(1)$} (0.1,2);
\draw[decoration={brace,mirror},decorate] (1.1,0.3) -- node[right=3pt] {$I(2)$} (1.1,1.3);
\draw[decoration={brace,mirror},decorate] (2.1,1) -- node[right=3pt] {$I(3)$} (2.1,1.5);
\path [fill=blue,opacity=0.2] (0,1) rectangle (4,1.25);
\node[blue] at (3.5,1.125) {$\cap_{i=1}^3 I(i)$};
\end{tikzpicture}
 \end{center}
\vspace{-0.4cm}
\caption{Illustration of \lepski with $M=5$. As $\cap_{i=1}^3 I(i)$ is
  non-empty but $I(4)$ does not intersect with $I(3)$, we select
  $\hat{i}=3$.}
\vspace{-0.4cm}
\label{fig:illustration}
\end{figure}
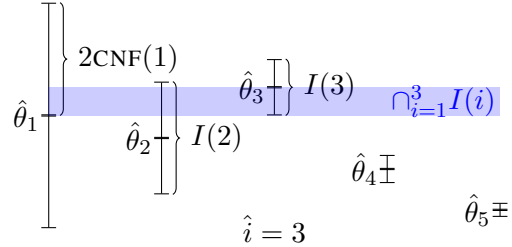
}

The intuition is to adopt an optimistic view of the bias function
$\bias(i)$. First observe that if $\bias(i) = 0$ then,
by~\pref{assum:concentration}, we must have $\theta^\star \in
I(i)$. 
Reasoning optimistically, it is
possible that we have $\bias(i) = 0$ for all $i \leq \hat{i}$, since
by the definition of $\hat{i}$ there exists a choice of $\theta^\star$
that is consistent with all intervals. As $\conf(\hat{i})$ is the
smallest among these, index $\hat{i}$ intuitively has lower error than
all previous indices. On the other hand, it is not possible to have
$\bias(\hat{i}+1) = 0$, since there is no consistent choice for
$\theta^\star$ and the bias is monotonically increasing. In fact, if
$\theta^\star \in I_{\hat{i}}$, then we must actually have
$\bias(\hat{i}+1) \geq \conf(\hat{i}+1)$, since the intervals have
width $4\conf(\cdot)$. Finally, since $\conf(\cdot)$ does not shrink
too quickly, all subsequent indices cannot be much better than
$\hat{i}$, the index we select. Of course, we may not have
$\theta^\star \in I_{\hat{i}}$, so this argument does not constitute a
proof of correctness, which is deferred to~\pref{app:proofs}.

\paragraph{Theoretical analysis.}
We now state the main guarantee.
\begin{theorem}
\label{thm:main}
Under~\pref{assum:concentration} and~\pref{assum:monotonicity}, we have that with probability at least $1-\delta$:
\begin{align*}
\abr{\hat{\theta}_{\hat{i}} - \theta^\star} &\leq 6(1+\monoconst^{-1})\min_i\cbr{\bias(i) + \conf(i)}.
\end{align*}
\end{theorem}
The theorem verifies that the index $\hat{i}$ satisfies an oracle
inequality as in~\pref{eq:guarantee_form}, with $\const =
6(1+\monoconst^{-1})$.  This is the best guarantee one could hope for, up
to the constant factor and the caveat that we are competing with the
error bound instead of the error, which we have already discussed.
For off-policy evaluation, we are not aware of any other approaches
that achieve any form of oracle inequality. The closest comparison is
the bound minimization approach of~\citet{su2019doubly}, which is
provably competitive only with unbiased indices (with $\bias(i) =
0$). However in finite sample, these indices could have high variance
and consequently worse performance than some biased estimator. In this
sense, the \lepski guarantee is much stronger.

While our main result gives a high probability absolute error bound,
it is common in the off-policy evaluation literature to instead
provide bounds on the mean squared error. Via a simple translation
from the high-probability guarantee, we can obtain a MSE bound here as
well. For this result, we use the notation $\conf(i;\delta)$ to
highlight the fact that the confidence bounds hold with probability
$1-\delta$.

\begin{corollary}[MSE bound]
\label{corr:mse}
In addition to~\pref{assum:concentration}
and~\pref{assum:monotonicity}, assume that $\theta^\star,
\hat{\theta}_i\in [0,R]$ a.s., $\forall i$, and that $\conf$ is
deterministic.\footnote{The restriction to deterministic confidence functions can easily be removed with another concentration argument.}  Then for any $\delta > 0$,
\begin{align*}
\EE (\hat{\theta}_{\hat{i}} - \theta^\star)^2 \leq \nicefrac{C}{\monoconst^2} \min_{i}\cbr{\bias(i)^2 + \conf(i;\delta)^2} +
R^2\delta,
\end{align*}
where $C > 0$ is a universal constant.\footnote{We have not attempted
  to optimize the constant, which can can be extracted from our proof
  in~\pref{app:proofs}.}
\end{corollary}
We state this bound for completeness but remark that it is typically
loose in constant factors because it is proven through a high
probability guarantee. In particular, we typically require $\conf(i) >
\sqrt{\var(i)}$ to satisfy~\pref{assum:concentration}, 
which is already loose in comparison with a more direct MSE bound.
Unfortunately, Lepski's principle cannot provide direct MSE bounds
without estimating the MSE itself, which is precisely the problem we
would like to avoid. On the other hand, the high probability guarantee
provided by~\pref{thm:main} is typically more practically meaningful.

\section{Application 1: continuous contextual bandits}
\label{sec:ccb}
For our first application, we consider a contextual bandit setting
with continuous action space, following~\citet{kallus2018policy}. Let
$\Xcal$ be a context space and $\Acal = [0,1]$ be a univariate
real-valued action space. There is a distribution $\Pcal$ over
context-reward pairs, which is supported on $(\Xcal, \Acal \to
[0,1])$. We have a stochastic logging policy $\pilog: \Xcal \to
\Delta(\Acal)$ which induces the distribution $\Dcal$ by generating
tuples $\{(x_i,a_i,r_i(a_i),\pilog(a_i))\}_{i=1}^n$, where $(x_i,r_i)
\sim \Pcal$, $a_i \sim \pilog(x_i)$, only $r_i(a_i)$ is observed, and
$\pilog(a_i)$ denotes the density value. This is a bandit setting as
the distribution $\Pcal$ specifies rewards for all actions, but only
the reward for the chosen action is available for estimation.

For off-policy evaluation, we would like to estimate the value of some
target policy $\pitarget$, which is given by $V(\pitarget) \defeq
\EE_{(x,r) \sim \Pcal,a\sim\pitarget(x)}\sbr{r(a)}$.
Of course, we do not have sample access to $\Pcal$ and must resort to
the logged tuples generated by $\pilog$. A standard off-policy
estimator in this setting is
\begin{align*}
\hat{V}_h(\pitarget) := \frac{1}{nh}\sum_{i=1}^n \frac{K( |\pitarget(x_i)- a_i|/h) r_i(a_i)}{\pilog(a_i)},
\end{align*}
where $K: \RR_+ \to \RR$ is a \emph{kernel function} (e.g., the boxcar
kernel $K(u) = \frac{1}{2} \one\{u \leq 1\}$).
This
estimator has appeared in recent
work~\citep{kallus2018policy,krishnamurthy2019contextual}.  The key
parameter is the bandwidth $h$, which modulates a bias-variance
tradeoff, where smaller bandwidths have lower bias but higher
variance.

\subsection{Theory}
For a simplified exposition, we instantiate our general framework 
when (1) $\pilog$ is the uniform logging
policy, (2) $K$ is the boxcar kernel, and (3) we assume that
$\pitarget(x) \in [\gamma_0,1-\gamma_0]$ for all $x$. These
simplifying assumptions help clarify the results, but they are not
fundamentally limiting.

Fix $\gamma \in (0,1)$ and let $\Hcal \defeq \{\gamma_0 \gamma^{M-i}:
1 \leq i \leq M\}$ denote a geometrically spaced grid of $M$ bandwidth
values. Let $\hat{\theta}_i \defeq \hat{V}_{h_i}(\pitarget)$. For the
confidence term,
in~\pref{app:proofs}, we show that we can set
\begin{align}
\conf(i) \defeq \sqrt{\frac{2 \log(2M/\delta)}{nh_i}} + \frac{2\log(2M/\delta)}{3nh_i} \label{eq:ccb_conf}
\end{align}
and this satisfies~\pref{assum:concentration}. With this form, it is not hard
to see that the second part of~\pref{assum:monotonicity} is also
satisfied, and so we obtain the following result.

\begin{theorem}
\label{thm:ccb}
Consider the setting above with uniform $\pilog$, boxcar kernel, and
$\Hcal$ as defined above. Let $\bias$ be any valid and monotone bias
function, and define \conf as
in~\pref{eq:ccb_conf}. Then~\pref{assum:concentration}
and~\pref{assum:monotonicity} are satisfied with $\monoconst=\gamma$,
so the guarantee in~\pref{thm:main} applies.

In particular, if $\gamma,\gamma_0$ are constants and
rewards are $L$-Lipschitz, for $\omega(\sqrt{\log(\log(n))/n}) \leq L
\leq O(n)$, then \begin{align*}
\abr{\hat{\theta}_{\hat{i}} - \theta^\star} \leq O\rbr{\rbr{\frac{L\log(\log(n)/\delta)}{n}}^{1/3}}
\end{align*}
with probability at least $1-\delta$, \emph{without knowledge of the
  Lipschitz constant $L$.}
\end{theorem}

For the second statement, we remark that if the Lipschitz constant
were known, the best error rate achievable is $O(
(L\log(1/\delta)/n)^{1/3})$. 
Thus, \lepski incurs almost no price for adaptation. We also note that
it is typically impossible to know this parameter in practice.

It is technical but not difficult to derive a more general result,
relaxing many of the simplifying assumptions we have made.
To this end, we provide a guarantee for non-uniform $\pilog$ in the
appendix. We do not pursue other extensions here, as the necessary
techniques are well-understood
(c.f.,~\citet{kallus2018policy,krishnamurthy2019contextual}).

\subsection{Experiments}
\label{sec:cb_exp}

We empirically evaluate using \lepski for bandwidth selection in a
synthetic environment for continuous action contextual bandits. We
summarize the experiments and findings here with detailed description
in~\pref{app:ccb_exp}.\footnote{Code for this section is publicly
  available at
  \url{https://github.com/VowpalWabbit/slope-experiments}.}

\ifthenelse{\equal{\version}{arxiv}}{
\begin{wraptable}{R}{0.46\textwidth}
\begin{center}
\vspace{-0.85cm}
\begin{tabular}{| c |}
\hline
Reward fn $\in \{ \textrm{quadratic, absolute value} \}$\\
\hline
Lipschitz const $\in \{0.1,0.3,1,3,10\}$\\
\hline
Kernel $\in \{\textrm{boxcar, Epanechnikov}\}$\\
\hline
$\pitarget \in \{\textrm{linear, tree}\}$\\
\hline
$\pilog \in \{\textrm{linear, tree}\}$\\
\hline
Randomization $\in \{\textrm{unif, friendly, adversarial}\}$\\
\hline
Samples $\in \{10^i : i \in \{1,2,3,4,5\}\}$\\
\hline
\end{tabular}
\vspace{-0.25cm}
\caption{Contextual Bandit experiment conditions.}
\vspace{-0.65cm}
\label{tab:cb_conds}
\end{center}
\end{wraptable}
}{
\begin{table}
\begin{center}
\begin{tabular}{| c |}
\hline
Reward fn $\in \{ \textrm{quadratic, absolute value} \}$\\
\hline
Lipschitz const $\in \{0.1,0.3,1,3,10\}$\\
\hline
Kernel $\in \{\textrm{boxcar, Epanechnikov}\}$\\
\hline
$\pitarget \in \{\textrm{linear, tree}\}$\\
\hline
$\pilog \in \{\textrm{linear, tree}\}$\\
\hline
Randomization $\in \{\textrm{uniform, friendly, adversarial}\}$\\
\hline
Samples $\in \{10^i : i \in \{1,2,3,4,5\}\}$\\
\hline
\end{tabular}
\caption{Contextual Bandit experimental conditions.}
\vspace{-0.4cm}
\label{tab:cb_conds}
\end{center}
\end{table}
}
\vspace{-0.2em}
\paragraph{The environment.}
We use a highly configurable synthetic environment,
which
allows for action spaces of arbitrary dimension, varying reward
function, reward smoothness, kernel, target, and logging
policies. In our experiments, we focus on $\Acal = [0,1]$. We vary all
other parameters, as summarized in~\pref{tab:cb_conds}.

The simulator prespecifies a mapping $x \mapsto a^\star(x)$ which is
the global maxima for the reward function. We train deterministic
policies by regressing from the context to this global maxima.
For the logging policy, we use two ``softening'' approaches
for randomization,
following~\citet{farajtabar2018more}. We use two regression models
(linear, decision tree), and two softenings in addition to uniform
logging, for a total of 5 logging and 2 target policies.

\vspace{-0.2em}
\paragraph{Methods.}
We consider 7 different choices of geometrically spaced bandwidths
$\Hcal \defeq \{2^{-i}:i \in [7]\}$. We evaluate the performance of
these fixed bandwidths in comparison with \lepski, which selects from
$\Hcal$. For \lepski, we simplify the implementation by replacing the
confidence function in~\pref{eq:ccb_conf}, with twice the empirical
standard deviation of the corresponding estimate. This approximation
is a valid asymptotic confidence interval and is typically sharper
than~\pref{eq:ccb_conf}, so we expect it to yield better practical
performance. We also manually enforce monotonicity of this confidence
function.

We are not aware of other viable baselines for this setting. In
particular, the heuristic method of~\citet{kallus2018policy} is too
computationally intensive to use at our scale.

\vspace{-0.2em}
\paragraph{Experiment setup.}
We have 1000 conditions determined by: logging
policy, target policy, reward functional form, reward smoothness,
kernel, and number of samples $n$. For each condition, we first obtain a
Monte Carlo estimate of the ground truth $V(\pitarget)$ by collecting
100k samples from $\pitarget$.  Then we collect $n$ trajectories from
$\pilog$ and evaluate the squared error of each method $(\hat{V} -
V(\pitarget))^2$. We perform 30 replicates of each condition with
different random seeds and calculate the correspond mean squared error
(MSE) for each method: $\mathrm{MSE} := \frac{1}{30}\sum_{i=1}^{30}
(\hat{V}_i - V(\pitarget))^2$ where $\hat{V}_i$ is the estimate on the
$i^{\textrm{th}}$ replicate.

\begin{figure*}
\includegraphics[width=\textwidth]{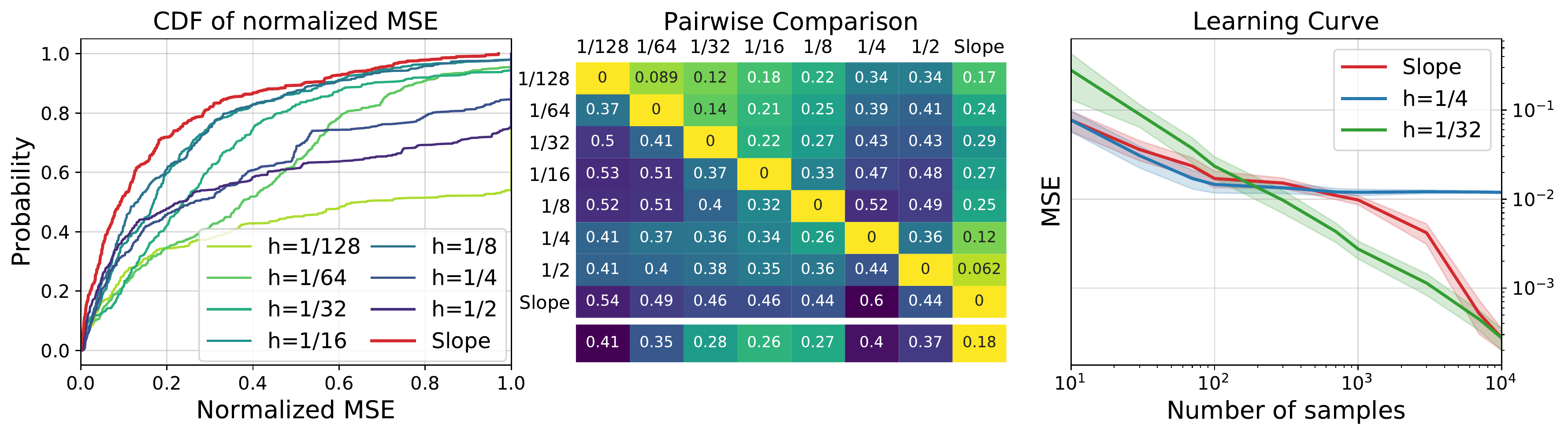}
\ifthenelse{\equal{\version}{arxiv}}{\vspace{-0.5cm}}{\vspace{-0.75cm}}
\caption{Experimental results for contextual bandits with continuous
  actions. Left: CDF of normalized MSE across all 480
  conditions. Normalization is by the worst MSE for that
  condition. Middle: Pairwise comparison matrix, entry $P_{i,j}$
  counts the fraction of conditions where method $i$ is statistically
  significantly better than method $j$, so larger numbers in the rows
  (or smaller numbers in the columns) is better. Right: asymptotic
  behavior of \lepski selecting between two bandwidths. }
\vspace{-0.25cm}
\label{fig:cb_exp}
\end{figure*}

\paragraph{Results.}
The left panel of~\pref{fig:cb_exp} aggregates results via the
empirical CDF of the normalized MSE, where we normalize by the worst
MSE in each condition. The point $(x,y)$ indicates that on
$y$-fraction of conditions the method has normalized MSE at most $x$,
so better methods lie in the top-left quadrant. We see that \lepski is
the top performer in comparison with the fixed bandwidths.

In the center panel, we record the results of pairwise comparisons
between all methods. Entry $(i,j)$ of this array is the fraction of
conditions where method $i$ is significantly better than method $j$
(using a paired $t$-test with significance level $0.05$). Better
methods have smaller numbers in their column, which means they are
typically not significantly worse than other methods. The final
row summarizes the results by averaging each column. In this
aggregation, \lepski also outperforms each individual fixed bandwidth,
demonstrating the advantage in data-dependent estimator selection.

Finally, in the right panel, we demonstrate the behavior of \lepski in
a single condition as $n$ increases. Here \lepski is only selecting
between two bandwidths $\Hcal\defeq\{1/4,1/32\}$. When $n$ is small,
the small bandwidth has high variance but as $n$ increases, the bias
of the larger bandwidth dominates. \lepski effectively
navigates this tradeoff, tracking $h=1/4$ when $n$ is small, and
switching to $h=1/32$ as $n$ increases.

\paragraph{Summary.}
\lepski is the top performer when compared with fixed bandwidths in
our experiments. This is intuitive as we do not expect a single
fixed bandwidth to perform well across all conditions. On the other
hand, we are not aware of other approaches for bandwidth selection in
this setting, and our experiments confirm that \lepski is a viable and
practically effective approach.

\section{Application 2: reinforcement learning}
\label{sec:rl}

Our second application is multi-step reinforcement learning (RL). We
consider episodic RL where the agent interacts with the environment in
episodes of length $H$. Let $\Xcal$ be a state space and $\Acal$ a
finite action space. In each episode, a trajectory $\tau \defeq
(x_1,a_1,r_1,x_2,a_2,r_2,\ldots,x_H,a_H,r_H)$ is generated where (1)
$x_1\in\Xcal$ is drawn from a starting distribution $P_0$, (2) rewards
$r_h \in \RR$ and next state $x_{h+1}\in\Xcal$ are drawn from a system
descriptor $(r_h,x_{h+1})\sim P_+(x_h,a_h)$ for each $h$ (with the
obvious definition for time $H$), and (3) actions $a_1,\ldots,a_H \in
\Acal$ are chosen by the agent. A policy $\pi : \Xcal \mapsto
\Delta(\Acal)$ chooses a (possibly stochastic) action in each state
and has value $V(\pi) \defeq \EE\sbr{\sum_{h=1}^H \gamma^{h-1} r_h
  \mid a_{1:H}\sim\pi}$, where $\gamma \in (0,1)$ is a discount
factor.
For normalization, we assume that rewards are in $[0,1]$ almost
surely.

\ifthenelse{\equal{\version}{arxiv}}{
\begin{wraptable}{R}{0.55\textwidth}
  \begin{center}
    \vspace{-0.4cm}
    \begin{tabular}{|c|c|c|c|c|} 
      \hline
      \text{Environment} & \text{GW} & \text{MC}  & \text{Graph} & \text{PO-Graph}\\
      \hline
      \text{Horizon} & 25 & 250 & 16 & 16\\
      \text{MDP} & \text{Yes} & \text{Yes} & \text{Yes} & \text{No} \\
      \text{Sto Env} & \text{Both} & \text{No} & \text{Yes} & \text{Yes}\\
      \text{Sto Rew} & \text{No} & \text{No} & \text{Both} &\text{Both}\\
      \text{Sparse Rew} & \text{No} & \text{No} & \text{Both} &\text{Both} \\
      \text{Model class} & \text{Tabular} & \text{NN} &\text{Tabular} &\text{Tabular}\\
      \text{Samples} & $2^{7:9}$ & $2^{8:10}$ & $2^{7:10}$ & $2^{7:10}$\\
      \text{\# of policies} & 5 & 4 & 2 & 2\\
            \hline
    \end{tabular}
    \vspace{-0.25cm}
    \caption{RL Environment Details}
    \vspace{-0.6cm}
    \label{tab:envs}
  \end{center}
\end{wraptable}
}{
\begin{table}[t]
  \begin{center}
    \begin{tabular}{|c|c|c|c|c|} 
      \hline
      \text{Environment} & \text{GW} & \text{MC}  & \text{Graph} & \text{PO-Graph}\\
      \hline
      \text{Horizon} & 25 & 250 & 16 & 16\\
      \text{MDP} & \text{Yes} & \text{Yes} & \text{Yes} & \text{No} \\
      \text{Sto Env} & \text{Both} & \text{No} & \text{Yes} & \text{Yes}\\
      \text{Sto Rew} & \text{No} & \text{No} & \text{Both} &\text{Both}\\
      \text{Sparse Rew} & \text{No} & \text{No} & \text{Both} &\text{Both} \\
      \text{Model class} & \text{Tabular} & \text{NN} &\text{Tabular} &\text{Tabular}\\
      \text{Samples} & $2^{7:9}$ & $2^{8:10}$ & $2^{7:10}$ & $2^{7:10}$\\
      \text{\# of policies} & 5 & 4 & 2 & 2\\
            \hline
    \end{tabular}
    \vspace{-0.4cm}
    \caption{RL Environment Details}
    \vspace{-0.5cm}
    \label{tab:envs}
  \end{center}
\end{table}
}

For off-policy evaluation, we have a dataset of $n$
trajectories
$\{(x_{i,1},a_{i,1},r_{i,1},\ldots,x_{i,H},a_{i,H},r_{i,H})\}_{i=1}^n$
generated by following some logging policy $\pilog$, and we would like
to estimate $V(\pitarget)$ for some other target policy. 
The importance weighting approach is also standard here,
and perhaps the simplest estimator is
\begin{align}
\hat{V}_{\mathrm{IPS}}(\pitarget) \defeq \frac{1}{n}\sum_{i=1}^n \sum_{h=1}^H \gamma^{h-1} \rho_{i,h} r_{i,h},\label{eq:rl_ips}
\end{align}
where $\rho_{i,h} \defeq \prod_{h'=1}^h\tfrac{\pitarget(a_{i,h'} \mid
  x_{i,h'})}{\pilog(a_{i,h'} \mid x_{i,h'})}$ is the step-wise
importance weight.  This estimator is provably unbiased under very
general conditions, but it suffers from high variance due to the
$H$-step product of density ratios.\footnote{We note that there are
  variants with improved variance. As our estimator selection question
  is somewhat orthogonal, we focus on the simplest estimator.} An
alternative approach is to directly model the value function using
supervised learning, as in a regression based dynamic programming
algorithm like Fitted Q
Evaluation~\citep{riedmiller2005neural,szepesvari2005finite}.
While these ``direct modeling'' approaches have very low variance,
they are typically highly biased because they rely on supervised
learning models that cannot capture the complexity of the environment.
Thus they lie on the other extreme of the bias-variance spectrum.

To navigate this tradeoff,~\citet{thomas2016data} propose a family of
\emph{partial importance weighting estimators}. To instantiate this
family, we first train a direct model $\hat{V}_{\mathrm{DM}} : \Xcal
\times [H] \to \RR$ to approximate $(x,h) \mapsto
\EE_{\pi}\sbr{\sum_{h'=h}^H\gamma^{h'-h}r_{h'} \mid x_h=x}$, for
example via Fitted Q Evaluation.  Then, the estimator is
\ifthenelse{\equal{\version}{arxiv}}{
\begin{align}
\label{eq:partial_ips}
\hat{V}_{\eta}(\pitarget) & \defeq \frac{1}{n}\sum_{i=1}^n
\sum_{h=1}^\eta \gamma^{h-1} \rho_{i,h} r_{i,h} +
\frac{1}{n}\sum_{i=1}^n 
\gamma^\eta\rho_{i,\eta}\vdm(x_{i,\eta+1},\eta+1),
\end{align}
}{
\begin{align}
\label{eq:partial_ips}
\hat{V}_{\eta}(\pitarget) & \defeq \frac{1}{n}\sum_{i=1}^n
\sum_{h=1}^\eta \gamma^{h-1} \rho_{i,h} r_{i,h} \\
& ~~~~ +
\frac{1}{n}\sum_{i=1}^n 
\gamma^\eta\rho_{i,\eta}\vdm(x_{i,\eta+1},\eta+1),\notag
\end{align}
}
The estimator has a parameter $\eta$ that governs a \emph{false
  horizon} for the importance weighting component. Specifically, we
only importance weight the rewards up until time step $\eta$ and we
complete the trajectory with the predictions from a direct modeling
approach. 
The model selection question here centers around choosing the
false horizon $\eta$ at which point we truncate the unbiased
importance weighted estimator. 

\subsection{Theory}
We instantiate our general estimator selection framework in this
setting.  Let $\hat{\theta}_i \defeq \hat{V}_{H-i+1}(\pitarget)$ for
$i \in \{1,\ldots,H+1\}$. Intuitively, we expect that the variance of
$\hat{\theta}_i$ is large for small $i$, since the estimator involves
a product of many density ratios. Indeed, in the appendix, we derive a
confidence bound and prove that it verifies our assumptions. The bound
is quite complicated so we do not display it here, but we refer the
interested reader to~\pref{eq:rl_dev} in~\pref{app:proofs}. The bound
is a Bernstein-type bound which incorporates both variance and range
information. We bound these as
\ifthenelse{\equal{\version}{arxiv}}{
\begin{align*}
\textrm{Variance}(\hat{V}_\eta(\pitarget)) \leq 3\vmax^2(1 + \sum_{h=1}^\eta \gamma^{2(h-1)}\pmax^h), \qquad 
\textrm{Range}(\hat{V}_\eta(\pitarget)) \leq 3\vmax (1+\sum_{h=1}^\eta \gamma^{h-1}\pmax^h),
\end{align*}
}{
\begin{align*}
\textrm{Variance}(\hat{V}_\eta(\pitarget)) \leq 3\vmax^2(1 + \sum_{h=1}^\eta \gamma^{2(h-1)}\pmax^h)\\
\textrm{Range}(\hat{V}_\eta(\pitarget)) \leq 3\vmax (1+\sum_{h=1}^\eta \gamma^{h-1}\pmax^h),
\end{align*}
}
where $\vmax\defeq (1-\gamma)^{-1}$ is the range of the value function
and $\pmax \defeq \sup_{x,a}\tfrac{\pitarget(a \mid x)}{\pilog(a \mid
  x)}$ is the maximum importance weight, which should be
finite. Equipped with these bounds, we can apply Bernstein's
inequality to obtain a valid confidence interval.\footnote{This yields
  a relatively concise deviation bound, but we note that it is not the
  sharpest possible.} Moreover, it is not hard to show that this
confidence interval is monotonic with $\monoconst \defeq
(1 + \gamma \pmax)^{-1}$. This yields the following theorem.
\begin{theorem}[Informal]
\label{thm:rl}
Consider the episodic RL setting with $\hat{\theta}_i \defeq
\hat{V}_{H-i+1}(\pitarget)$ defined in~\pref{eq:partial_ips}. Let
$\bias$ be any valid and monotone bias function. Then with $\conf(i)$
as in~\pref{eq:rl_dev} in the appendix,~\pref{assum:concentration} and 
~\pref{assum:monotonicity} with
$\monoconst \defeq(1+\gamma\pmax)^{-1}$ hold, so~\pref{thm:main} applies.
\end{theorem}
A more precise statement is provided in~\pref{app:proofs}, and we
highlight some salient details here. First, our analysis actually
applies to a doubly-robust variant of the estimator $\hat{V}_\eta$, in
the spirit of~\citep{jiang2015doubly}.  Second, $\bias(i) \defeq
\tfrac{\gamma^{H-i+1}-\gamma^H}{1-\gamma}$ is valid and monotone, and
can be used to obtain a concrete error bound.
However, 
the oracle inequality yields a stronger conclusion, since it applies
for \emph{any} valid and monotone bias function. This universality is
particularly important when using the doubly robust variant, since it
is typically not possible to sharply bound the bias.

The closest comparison is \magic~\citep{thomas2016data}, which is
strongly consistent in our setting. However, it does not satisfy any
oracle inequality and is dominated by \lepski in experiments.

\subsection{Experiments}
We evaluate \lepski in RL environments spanning 106 different
experimental conditions.
We also compare with previously proposed estimators and assess robustness under various conditions.
Our experiments closely follow the setup
of~\citet{voloshin2019empirical}. Here we provide an overview of the
experimental setup and highlight the salient differences from
theirs. All experimental details are
in~\pref{app:rl_exp}.\footnote{Code for this section is available at \url{https://github.com/clvoloshin/OPE-tools}.}

\paragraph{The environments.}
We use four RL environments: Mountain Car, Gridworld, Graph, and
Graph-POMDP (abbreviated MC, GW, Graph, and PO-Graph). All four
environments are from~\citet{voloshin2019empirical}, and they provide
a broad array of environmental conditions, varying in terms of horizon
length, partial observability, stochasticity in dynamics,
stochasticity in reward, reward sparsity, whether function
approximation is required, and overlap between logging and target
policies. Logging and target policies are
from~\citet{voloshin2019empirical}. A summary of the environments and
their salient characteristics is displayed in~\pref{tab:envs}.

\begin{figure*}[tb]
\includegraphics[width=\textwidth]{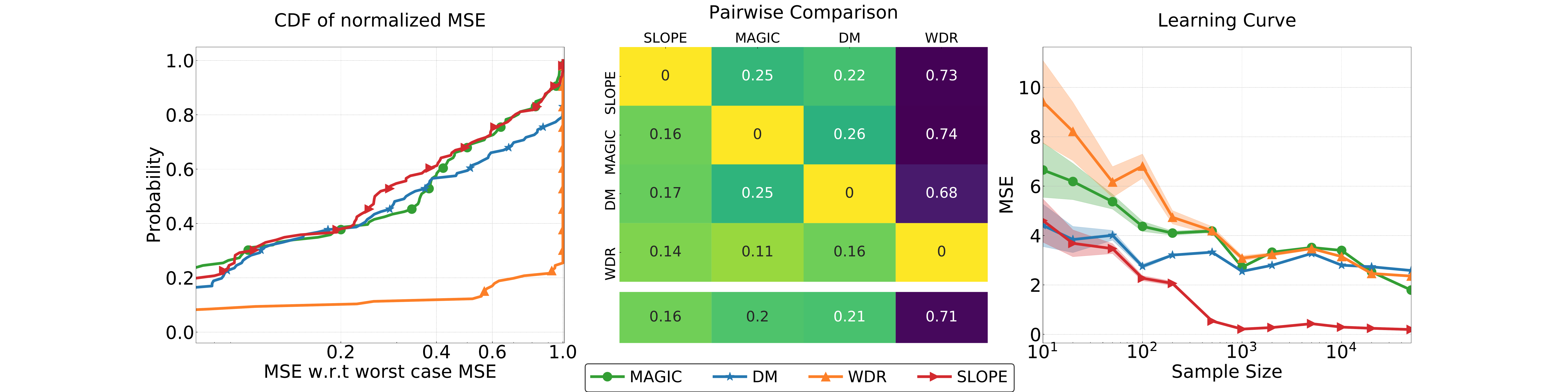}
\vspace{-0.5cm}
\caption{Left: Cumulative distribution function of the normalized MSE for all conditions, Middle: Pairwise comparison matrix $P$ for the methods, over all conditions. Element $P_{ij}$ denotes the percentage of times that method $i$ outperforms method $j$. The last row shows the column average for each method, the lower the better. Right: Learning Curve for the Hybrid domain.}
\label{fig:rl_fig}
\vspace{-0.25cm}
\end{figure*}

\paragraph{Methods.}
We compare four estimators: the direct model (DM), a self-normalized
doubly robust estimator (WDR), (c) \magic, and (d) \lepski. All four
methods use the same direct model, which we train either by Fitted Q
Evaluation or by $Q^\pi(\lambda)$~\citep{munos2016safe}, following the
guidelines in~\citet{voloshin2019empirical}. The doubly robust
estimator is the most competitive estimator in the family of
full-trajectory importance weighting. It is similar
to~\pref{eq:rl_ips}, except that the direct model is used as a control
variate and the normalizing constant $n$ is replaced with the sum of
importance weights. \magic, as we have alluded to, is the only other
estimator selection procedure we are aware of for this setting. It
aggregates partial importance weighting estimators to optimize a
surrogate for the MSE. For \lepski, we use twice the empirical
standard deviation as the confidence function, which is asymptotically
valid and easier to compute.

We do not consider other baselines for two reasons. First, DM, WDR,
and \magic span the broad estimator categories (importance weighted,
direct, hybrid) within which essentially all estimators
fall. Secondly, many other estimators have hyperparameters that must
be tuned, and we believe \lepski will also be beneficial when used in
these contexts.

\paragraph{Experiment Setup.}
We have 106 experimental conditions determined by
environment, stochasticity of dynamics and reward,
reward sparsity, logging policy, target policy, and number of
trajectories $n$.
For each condition, we calculate the MSE for each method by averaging over 100 replicates.

\paragraph{Results.}
In the left panel of~\pref{fig:rl_fig}, as in~\pref{sec:cb_exp}, we
first visualize the aggregate results via the cumulative distribution
function (CDF) of the normalized MSE in each condition (normalizing
by the worst performing method in each condition). As a reminder, the
figure reads as follows: for each $x$ value, the corresponding $y$
value is the fraction of conditions where the estimator has normalized
MSE at most $x$. In this aggregation, we see that WDR has the worst
performance, largely due to intolerably high variance. \magic and DM
are competitive with each other with \magic having a slight
edge. \lepski appears to have the best aggregate performance; for the
most part its CDF dominates the others.

In the central panel, we display an array of statistical comparisons
between pairs of methods. As before, entry $(i,j)$ of this array is
computed by counting the fraction of conditions where method $i$ beats
$j$ in a statistically significant manner (we use paired $t$-test on
the MSE with significance level $0.05$). The column-wise averages are
also displayed.

In this aggregation, we see clearly that \lepski dominates the three
other methods. First, \lepski has column average that is smaller than
the other methods. More importantly, \lepski is favorable when
compared with each other method individually. For example, \lepski is
(statistically) significantly worse than \magic on 16\% of the
conditions, but it is significantly better on 25\%. Thus, this
visualization clearly demonstrates that \lepski is the best performing
method in aggregate across our experimental conditions.

Before turning to the final panel of~\pref{fig:rl_fig}, we
recall~\pref{fig:intro_exp}, where we display results for two specific
conditions. Here, we see that \lepski outperforms or is statistically
indistinguishable from the best baseline, regardless of whether direct
modeling is better than importance weighting! We are not aware of any
selection method that enjoys this property.

\paragraph{Learning curves.}
The final panel of~\pref{fig:rl_fig} visualizes the performance of the
four methods as the sample size increases. Here we consider the Hybrid
domain from~\citet{thomas2016data}, which is designed specifically to
study the performance of partial importance weighting estimators. The
domain has horizon 22, with partial observability in the first two
steps, but full observability afterwards. Thus a (tabular) direct
model is biased since it is not expressive enough for the first two
time steps, but $\hat{V}_2(\pitarget)$ is a great estimator since the
direct model is near-perfect afterwards.

The right panel of~\pref{fig:rl_fig} displays the MSE for each method
as we vary the number of trajectories, $n$ (we perform 128
replicates and plot bars at $\pm 2$ standard errors). We see that when $n$ is small, DM
dominates, but its performance does not improve as $n$ increases due
to bias. Both WDR and \magic catch up as $n$ increases, but \lepski is
consistently competitive or better across all values of $n$,
outperforming the baselines by a large margin. Indeed, this is because
\lepski almost always chooses the optimal false horizon index of
$\eta=2$ (e.g., 90\% of the replicates when $n=500$).

\paragraph{Summary.}
Our experiments show that \lepski is competitive, if not the best,
off-policy evaluation procedure among \lepski, \magic, DM, and WDR. We
emphasize that \lepski is not an estimator, but a selection procedure
that in principle can select hyperparameters for many estimator
families. Our experiments with the partial importance weighting family
are quite convincing,
and we believe this demonstrates the potential for \lepski when used
with other estimator families for off-policy evaluation in
RL.\footnote{In settings where straightforward empirical variance
  estimates are not available, the bootstrap may provide an
  alternative approach for constructing the $\conf$
  function. Experimenting with such estimators is a natural future direction.}

\section{Discussion}
\label{sec:discussion}
In summary, this paper presents a new approach for estimator selection
in off-policy evaluation, called \lepski. The approach applies quite
broadly; in particular, by appropriately spacing hyperparameters, many
common estimator families can be shown to satisfy the assumptions for
\lepski. To demonstrate this, we provide concrete instantiations in
two important applications.  Our theory yields, to our knowledge, the
first oracle-inequalities for off-policy evaluation in RL. Our
experiments demonstrate strong empirical performance, suggesting that
\lepski may be useful in many off-policy evaluation contexts.
 
\section*{Acknowledgements}
We thank Mathias Lecuyer for comments on a preliminary version of this
paper.

\clearpage

\appendix
\section{Proofs}
\label{app:proofs}

\subsection{Proofs for~\pref{sec:general}}
\begin{proof}[Proof of~\pref{thm:main}]
The proof is similar to that of Corollary 1 in~\citep{mathe2006lepskii}.
Define $\tilde{i} = \max\{i: \bias(i) \leq \conf(i)\}$. The proof is
composed of two steps: first we show that we are competitive with
$\tilde{i}$ and then we show that $\tilde{i}$ is competitive with the
best index.

\paragraph{Competing with $\tilde{i}$.}
Observe that since $\bias$ is monotonically increasing, and $\conf$ is
monotonically decreasing, for $i \leq \tilde{i}$ we have
\begin{align*}
\bias(i) \leq \bias(\tilde{i}) \leq \conf(\tilde{i}) \leq \conf(i).
\end{align*}
Therefore, for $i \leq \tilde{i}$
\begin{align*}
\abr{\hat{\theta}_i - \theta^\star} \leq \bias(i) + \conf(i) \leq 2\conf(i).
\end{align*}
This implies that $\theta^\star \in I_i$ for all $i \leq
\tilde{i}$. 

As a consequence, the definition of our chosen index $\hat{i}$ implies
that $\hat{i} \geq \tilde{i}$, which in turn implies that
$I_{\tilde{i}} \cap I_{\hat{i}} \ne \emptyset$. So, there exists $x
\in I_{\tilde{i}} \cap I_{\hat{i}}$ such that $|x -
\hat{\theta}_{\tilde{i}}| \leq 2 \conf(\tilde{i})$ and $|x -
\hat{\theta}_{\hat{i}}| \leq 2 \conf(\hat{i})$. As we know that
$\theta^\star \in I_{\tilde{i}}$, we get
\begin{align}
\label{eq:ihat_vs_itilde}
|\hat{\theta}_{\hat{i}} - \theta^\star| \leq |\hat{\theta}_{\hat{i}} - x| + |x - \hat{\theta}_{\tilde{i}}| + |\hat{\theta}_{\tilde{i}} - \theta^\star| \leq 2 \conf(\hat{i}) + 2 \conf(\tilde{i}) + 2 \conf(\tilde{i}) \leq 6 \conf(\tilde{i}).
\end{align}

\paragraph{Comparing $\tilde{i}$ to $i^\star$.}
Define $i^\star \defeq \argmin_i\{ \bias(i) + \conf(i)\}$ which is the
index we actually want to compete with in our guarantee. If we compare
$\tilde{i}$ with $i^\star$, then by the above argument we can
translate to $\hat{i}$. For this, we consider two cases:

If $i^\star \leq \tilde{i}$, then by definition of $\tilde{i}$, we have
\begin{align*}
\bias(\tilde{i}) + \conf(\tilde{i}) \leq 2 \conf(\tilde{i}) \leq 2 \conf(i^\star) \leq 2 (\conf(i^\star) + \bias(i^\star)), 
\end{align*}
so we are a factor of $2$ worse. 

On the other hand, if $i^\star > \tilde{i}$ then
by~\pref{assum:monotonicity} and the optimality condition for
$\tilde{i}$
\begin{align*}
\monoconst \conf(\tilde{i}) \leq \conf(\tilde{i}+1) \leq \bias(\tilde{i}+1) \leq \bias(i^\star).
\end{align*}
This implies
\begin{align*}
\bias(\tilde{i}) + \conf(\tilde{i}) \leq (1+\nicefrac{1}{\monoconst}) \bias(i^\star). 
\end{align*}
As $\monoconst \leq 1$, this bound dominates the previous case, and together, with~\pref{eq:ihat_vs_itilde} we have
\begin{align*}
\abr{\hat{\theta}_{\hat{i}}- \theta^\star} \leq 6\times\conf(\tilde{i}) \leq 6(1+\nicefrac{1}{\monoconst}) \min_{i \in [M]} \cbr{\bias(i) + \conf(i)}. \tag*\qedhere
\end{align*}
\end{proof}

\begin{proof}[Proof of~\pref{corr:mse}]
For the MSE calculation, we simply need to translate from the high
probability guarantee to the MSE, which is not difficult
under~\pref{assum:concentration}. 
In particular, fix $\delta$ and let $\Ecal$ be the event that all
confidence bounds are valid, which holds with probability $1-\delta$,
then we have
\begin{align*}
\EE (\hat{\theta}_{\hat{i}} - \theta^\star)^2 &= \EE (\hat{\theta}_{\hat{i}} - \theta^\star)^2 \one\cbr{\Ecal} + \EE (\hat{\theta}_{\hat{i}} - \theta^\star)^2 \one\cbr{\bar{\Ecal}}\\
& \leq \EE (\hat{\theta}_{\hat{i}} - \theta^\star)^2 \one\cbr{\Ecal} + R^2 \delta\\
& \leq \EE \one\cbr{\Ecal} \rbr{6 (1+\nicefrac{1}{\monoconst}) \EE \min_{i \in [M]} \cbr{\bias(i) + \conf(i;\delta)}}^2 + R^2\delta\\
& \leq 72 (1+\nicefrac{1}{\monoconst})^2 \min_{i \in [M]}\cbr{\bias(i)^2 + \conf(i;\delta)^2} + R^2\delta.
\end{align*}
Here in the first line we are introducing the event $\Ecal$ and its
complement. In the second, we use that $\hat{\theta}_i, \theta^\star
\in [0,R]$ almost surely and that $\PP[\bar{\Ecal}] \leq \delta$
according to~\pref{assum:concentration}. In the third line, we
apply~\pref{thm:main}, which holds under event $\Ecal$. The final step
uses the simplification that $(a+b)^2 \leq 2a^2 + 2b^2$.
\end{proof}

\subsection{Proofs for~\pref{sec:ccb}}
\begin{proof}[Proof of~\pref{thm:ccb}]
Let us first verify that the confidence function specified
in~\pref{eq:ccb_conf} satisfy~\pref{assum:concentration}. We will
apply Bernstein's inequality, which requires variance and range
bounds. For the variance, a single sample satisfies
\begin{align*}
\Var\rbr{\frac{K(|\pitarget(x) - a|/h) r(a)}{h \times \pilog(a\mid x)}} 
& \leq \frac{1}{h} \EE\sbr{\frac{K(|\pitarget(x) - a|/h)^2}{h\times\pilog^2(a\mid x)} \mid a \sim \pilog(\cdot \mid x)} \leq \frac{1}{2h},
\end{align*}
where we first use that the variance is upper bounded by the second
moment, and then we use that $\pilog$ is uniform and the boxcar kernel
is at most $1/2$. Finally, we use that by a change of variables
$K(\cdot/h)/h$ integrates to $1$. Note that we are using that
$\pitarget \in [\gamma_0,1-\gamma_0]$, as we are integrating over
the support of $\pilog$.

For the range, we have
\begin{align*}
\sup \frac{K(|\pitarget(x) - a|/h)r(a)}{h\times\pilog(a\mid x)} \leq \frac{1}{2h}.
\end{align*}
Therefore, Bernstein's inequality gives that with probability $1-\delta$ we have
\begin{align*}
\abr{\hat{V}_h(\pitarget) - \EE \hat{V}_h(\pitarget)} \leq \sqrt{\frac{\log(2/\delta)}{nh}} + \frac{\log(2/\delta)}{3nh},
\end{align*}
and the first claim follows by a union bound. 

Monotonicity is also easy to verify with this definition of
$\conf$. In particular, since $h_i = \gamma h_{i+1}$ and $\gamma < 1$,
we immediately have that
\begin{align*}
\gamma \conf(i) = \gamma\sqrt{\frac{\log(2M/\delta)}{nh_i}} + \gamma\frac{\log(2M/\delta)}{3nh_i} = \gamma \sqrt{\frac{\log(2M/\delta)}{n\gamma h_{i+1}}} + \frac{\log(2M/\delta)}{3nh_{i+1}} \leq \conf(i+1)
\end{align*}
Clearly $\conf(i+1)\leq\conf(i)$, and so~\pref{assum:monotonicity}
holds. This verifies that we may apply~\pref{thm:main}.

For the last claim, if the rewards are $L$-Lipschitz, then we claim we
can set $\bias(i) = Lh_{i}$. To see why, observe that
\begin{align*}
\abr{\EE V_{h}(\pitarget) - V(\pitarget)} &= \abr{\EE_{(x,r), a \sim \pilog(x)} \frac{K(|\pitarget(x) - a|/h) r(a) }{h} - r(\pitarget(x))}\\
& = \abr{\EE_{(x,r)} \int_{a'} \frac{\one\cbr{|\pitarget(x) - a| \leq h} r(a)}{2h} - r(\pitarget(x))}\\
& = \abr{\EE_{(x,r)} \int_{a'} \frac{\one\cbr{|\pitarget(x) - a| \leq h} (r(a)- r(\pitarget(x)))}{2h}}
\leq Lh.
\end{align*}
Clearly this bias bound is monotonic. To apply~\pref{thm:main}, it is
better to first simplify the confidence function. Observe that as
$\theta^\star,\bar{\theta}_i \in [0,1]$, it is always better to clip
the estimates $\hat{\theta}_i$ to lie in $[0,1]$. This has no bearing
on the bias and only improves the deviation term, and in particular
allows us to replace $\conf(i)$ with $\min\{\conf(i),1\}$. This leads
to a further simplification:
\begin{align*}
\sqrt{\frac{\log(2M/\delta)}{nh_i}} + \frac{\log(2M/\delta)}{3nh_i} \leq 1 \Rightarrow \sqrt{\frac{\log(2M/\delta)}{nh_i}} + \frac{\log(2M/\delta)}{3nh_i} \leq \frac{4}{3}\sqrt{\frac{\log(2M/\delta)}{nh_i}}
\end{align*}

Therefore we may replace $\conf$ with this latter function and
by~\pref{thm:main} we guarantee that with probability at least
$1-\delta$
\begin{align*}
\abr{\hat{\theta}_{\hat{i}} - \theta^\star} \leq 6(1+\gamma^{-1})\min_i\cbr{ L h_i + \frac{4}{3}\sqrt{\frac{\log(2M/\delta)}{nh_i}}}.
\end{align*}
The optimal choice for $h$ is $h^\star \defeq
\rbr{\frac{4}{3L}\sqrt{\frac{\log(2M/\delta)}{n}}}^{2/3}$, which will
in general not be in our set $\Hcal$. However, if we use this choice
for $h$, the error rate is $O( (L/n)^{1/3})$, and since we know that
$\theta^\star \in [0,1]$, if $L > n$ then this error
guarantee is trivial. In other words, the maximum value of $L$ that we
are interested in adapting to is $L_{\max} = n$. This will be useful
in setting the number of models to search over $M$.

To set $M$, we want to ensure that there exists some $h_i$ such that
$h_i \leq h^\star \leq h_{i+1}$. We first verify the first inequality, which requires that
\begin{align*}
\gamma_0 \gamma^M \leq h^\star \defeq \rbr{\frac{4}{3L}\sqrt{\frac{\log(2M/\delta)}{n}}}^{2/3}
\end{align*}
We will always take $M \geq 2$, which implies that $\log(2M/\delta) \geq 1$. Then, since we are only interested in $L \leq n$, a sufficient condition here is
\begin{align*}
\gamma_0 \gamma^M \leq \frac{4}{3}^{2/3} n^{-1} \Rightarrow M \geq \frac{C_{\gamma_0} \log(n)}{\log(1/\gamma_0)},
\end{align*}
where $C_{\gamma_0}$ is a constant that only depends on
$\gamma_0$. The upper bound $h^\star \leq h_{i+1}$ is satisfied as
soon as $n$ is large enough, provided that $L \geq
\omega(\sqrt{\log(\log(n))/n})$, which we are assuming. Thus we know
that there is $i^\star$ such that $h_{i^\star} \leq h^\star \leq
h_{i^\star}/\gamma$, and using this choice, we have
\begin{align*}
\abr{\hat{\theta}_{\hat{i}} - \theta^\star} &\leq 6 (1+\gamma^{-1}) \cbr{ L h_{i^\star} + \frac{4}{3} \sqrt{\frac{\log(2M/\delta)}{nh_{i^\star}}}}
 \leq 6(1+\gamma^{-1}) \cbr{ L h^\star + \frac{4}{3} \sqrt{\frac{\log(2M/\delta)}{\gamma nh^\star}}}\\
& \leq 6 (1+\gamma^{-1}) \cdot \frac{c_1}{\sqrt{\gamma}} (L\log(2M/\delta)/n)^{1/3} \leq C_{\gamma,\gamma_0} (L\log(\log(n)/\delta)/n)^{1/3},
\end{align*}
where $C_{\gamma,\gamma_0}$ is a constant that depends only on
$\gamma,\gamma_0$.
\end{proof}

Note that if $\pilog$ is non-uniform, but satisfies $\inf_{x,a}
\pilog(a \mid x) \geq \pmin$, then very similar arguments apply. In
particular, we have that both variance and range are bounded by
$\frac{1}{2h\pmin}$, and some Bernstein's inequality in this case
yields
\begin{align*}
\abr{\hat{V}_h(\pitarget) - \EE \hat{V}_h(\pitarget)} \leq \sqrt{\frac{\log(2/\delta)}{nh\pmin}} + \frac{\log(2/\delta)}{3nh\pmin},
\end{align*}
Monotonicity follows from the same calculation as before and the
clipping trick yields a more interpretable final bound, which holds
with probability at least $1-\delta$, of
\begin{align*}
\abr{\hat{\theta}_{\hat{i}} - \theta^\star} \leq 6(1+\gamma^{-1})\min_i\cbr{ L h_i + \frac{4}{3}\sqrt{\frac{\log(2M/\delta)}{nh_i\pmin}}}.
\end{align*}
The remaining calculation for $M$ is analogous, since this bound is
identical to the previous one with $n$ replaced by $n\pmin$. Thus, we
obtain a final bound of
$C_{\gamma,\gamma_0}(L\log(\log(n/\delta))/(n\pmin))^{1/3}$.

\subsection{Proofs for~\pref{sec:rl}}

We first develop and state the more precise version
of~\pref{thm:rl}. We introduce the doubly robust version of the
partial importance weighting estimator. As it is the empirical average
over $n$ trajectories, here we will focus on a single trajectory
$(x_1,a_1,r_1,\ldots,x_H,a_H,r_H)$ sampled by following the logging
policy $\pilog$. 

Define $\vdr^0 \defeq 0$ and
\begin{align*}
\vdr^{H+1-h} \defeq \hat{V}(x_h) + p_h(r_h + \gamma \vdr^{H-t} - \hat{Q}(x_h,a_h)), \qquad p_h \defeq \frac{\pitarget(a_h\mid x_h)}{\pilog(a_h\mid x_h)}.
\end{align*}
where $\hat{Q}$ is the direct model, trained via supervised learning,
and $\hat{V}(x) = \EE_{a\sim\pitarget(x)}\hat{Q}(x,a)$.  The full
horizon doubly-robust estimator is $\vdr\defeq \vdr^H$. To define the
$\eta$-step partial estimator, let $\vdm^\eta \defeq \rho_\eta
\hat{Q}(x_{\eta+1},\pitarget(x_{\eta+1}))$, which is an estimate of
$\EE_{\pitarget}\sbr{V(x_{\eta+1})}$. Set $\vdm^H \defeq 0$. Then for
a false horizon $\eta$, we define a similar recursion
\begin{align*}
\hat{V}^{H+1-h}_\eta \defeq 
\left\{
\begin{aligned}
\hat{V}(x_h) + p_h(r_h + \gamma V_\eta^{H-h} - \hat{Q}(x_h,a_h)) & \textrm{ if } 1 \leq h < \eta\\
\hat{V}(x_h) + p_h(r_h + \gamma \vdm^\eta - \hat{Q}(x_h,a_h)) & \textrm{ if } h = \eta.
\end{aligned}
\right.
\end{align*}
The doubly robust variant of the $\eta$-step partial importance
weighted estimator is $\hat{V}_\eta^H$. We also define $\hat{V}_0^H =
\vdm^0$ which estimates $\EE\sbr{V(x_1)}$. Observe that if in the
definition of $\vdr$, we take $\hat{V},\hat{Q} \equiv 0$ then we
obtain the estimator in~\pref{eq:partial_ips}.

Define $\Delta \defeq \log(2(H+1)/\delta)$, $\vmax\defeq
(1-\gamma)^{-1}$ and recall that $\pmax\defeq
\max_{x,a}\tfrac{\pitarget(a |x)}{\pilog(a|x)}$. Then define
\begin{align*}
\bias(i) &\defeq \frac{\gamma^{H-i+1}-1}{1-\gamma}\\
\conf(i) &\defeq \sqrt{\frac{6\vmax^2\rbr{1+\sum_{h=1}^{H-i+1}\gamma^{2(h-1)}\pmax^h}\Delta}{n}} + \frac{6\vmax\rbr{1+\sum_{h=1}^{H-i+1}\gamma^{h-1}\pmax^h}\Delta}{3n}
\end{align*}

With these definitions, we know state the theorem
\begin{theorem}[Formal version of~\pref{thm:rl}]
\label{thm:rl_formal}
In the episodic reinforcement learning setting with discount factor
$\gamma$, consider the doubly robust partial importance weighting
estimators $\hat{\theta}_i \defeq \hat{V}^H_{H-i+1}(\pitarget)$ for $i
\in \{1,\ldots,H+1\}$. Then $\bias$ and $\conf$ are valid and
monotone, with $\kappa \defeq (1+\gamma\pmax)^{-1}$. 
\end{theorem}

\begin{proof}[Proof of~\pref{thm:rl_formal}]
We now turn to the proof. 

\paragraph{Bias analysis.}
By repeatedly applying the tower property, the expectation for
$\hat{V}_\eta^H$ is
\begin{align*}
\EE\sbr{\hat{V}_\eta^H} &= \EE_{\pilog}\sbr{\hat{V}(x_1) + p_1(r_1 + \gamma \hat{V}_\eta^{H-1} - \hat{Q}(x_1,a_1)}\\
& = \EE_{x_1}\sbr{ \hat{V}(x_1) + \EE_{a_1\sim\pilog(x_1),a_{2:H}\sim\pilog}\sbr{p_1(r_1 + \gamma \hat{V}_\eta^{H-1} - \hat{Q}(x_1,a_1)) \mid x_1}}\\
& = \EE_{x_1}\sbr{\hat{V}(x_1) + \EE_{a_1\sim\pitarget(x_1),a_{2:H}\sim\pilog}\sbr{r_1 + \gamma \hat{V}_\eta^{H-1} - \hat{Q}(x_1,a_1) \mid x_1}}\\
& = \EE_{x_1,a_1 \sim\pitarget(x_1)}[r] + \gamma \EE_{x_2\sim\pitarget,a_{2:H}\sim\pilog}\sbr{\hat{V}_\eta^{H-1}}\\
& = ...\\
& = \EE_{\pitarget}\sbr{\sum_{h=1}^\eta \gamma^{h-1} r} + \gamma^\eta \EE_{x_{\eta+1}\sim\pitarget}\sbr{\vdm^\eta}.
\end{align*}
Here, we use that $p_1$ is the one-step importance weight, so it
changes the action distribution from $\pilog$ to $\pitarget$. We also
use the relationship between the direct models $\hat{Q}$ and
$\hat{V}$.  Therefore, the bias is
\begin{align}
\label{eq:rl_bias_exact}
\abr{\EE\sbr{\hat{V}_\eta^H} - V(\pitarget)} = \gamma^\eta
\abr{\vdm^\eta - \EE_{\pitarget}\sbr{V(x_{\eta+1})}} \leq \frac{\gamma^\eta-\gamma^{H}}{1-\gamma} \eqdef \bias(H-\eta+1)
\end{align}
The first identity justifies are choice of $\vdm^\eta$ which attempts
to minimize this bias using the direct model. The inequality here
follows from the fact that rewards are in $[0,1]$, which implies that
values at time $\eta+1$ are in
$\sbr{0,\frac{1-\gamma^{H-\eta}}{1-\gamma}}$. As $\gamma \in (0,1)$,
clearly we have that $\bias(i)$ is monotonically increasing with $i$
increasing. Thus this bias bound is valid.

\paragraph{Variance analysis.}
For the variance calculation, let $\EE_h\sbr{\cdot},\Var_h(\cdot)$
denote expectation and variance conditional on all randomness
\emph{before} time step $h$. Adapting Theorem 1
of~\citet{jiang2015doubly} the variance for $1\leq h < \eta$ is given
by the recursive formula:
\ifthenelse{\equal{\version}{arxiv}}{
\begin{align*}
\Var_h(\hat{V}_\eta^{H+1-h}) &= \Var_h(\EE\sbr{\hat{V}_\eta^{H+1-h}\mid x_h}) + \EE_h\sbr{\Var(p_h\Delta(x_h,a_h)\mid x_h)}\\
 & ~~~~~~ + \EE_h\sbr{p_h^2\Var(r_h)} + \EE_h\sbr{ \gamma^2p_h^2\Var_{h+1}(\hat{V}_\eta^{H-h})},
\end{align*}
}{
\begin{align*}
\Var_h(\hat{V}_\eta^{H+1-h}) = \Var_h(\EE\sbr{\hat{V}_\eta^{H+1-h}\mid x_h}) + \EE_h\sbr{\Var(p_h\Delta(x_h,a_h)\mid x_h)} + \EE_h\sbr{p_h^2\Var(r_h)} + \EE_h\sbr{ \gamma^2p_h^2\Var_{h+1}(\hat{V}_\eta^{H-h})},
\end{align*}
}
where $\Delta(x_h,a_h)\defeq \hat{Q}(x_h,a_h) - Q(x_h,a_h)$. For $h =
\eta$ it is identical, except that in the last term we use $\vdm^\eta$
instead of $\hat{V}_\eta^{H-\eta}$ (which is not defined).

Unrolling the recursion, the full expression for the variance is
\begin{align*}
\Var(\hat{V}_\eta^H) 
& = \sum_{h=1}^\eta \EE\sbr{\gamma^{2(h-1)}\rho^2_{h-1}\Var_h(\EE\sbr{\hat{V}_\eta^{H+1-h}\mid x_h})}\\
& + \sum_{h=1}^\eta\EE\sbr{\gamma^{2(h-1)}\rho^2_{h-1}\EE_h\sbr{\Var(p_h\Delta(x_h,a_h) \mid x_h)}}\\
& + \sum_{h=1}^\eta \EE\sbr{\gamma^{2(h-1)}\rho^2_{h-1} \EE_h\sbr{p_h^2\Var(r_h)}}\\
& + \EE\sbr{\gamma^{2\eta}\rho_\eta^2 \Var_{\eta+1}(\vdm^\eta)}.
\end{align*}

For the variance bound, we do not attempt to obtain the sharpest bound
possible.  Instead, we use the following facts: (1) rewards are in
$[0,1]$, (2) all values, value estimates, and $Q$ are at most
$(1-\gamma)^{-1}\eqdef \vmax$, and (3) for a random variable $X$ that
is bounded by $B$ almost surely, we have $\Var(X) \leq B^2$. Using these
facts in each term gives
\begin{align*}
\Var(\hat{V}_\eta^H) 
& \leq \sum_{h=1}^\eta \EE\sbr{ \gamma^{2(h-1)}\rho^2_{h-1}\vmax^2} + \sum_{h=1}^\eta \EE\sbr{\gamma^{2(h-1)}\rho_h^2\vmax^2} + \sum_{h=1}^\eta \EE\sbr{\gamma^{2(h-1)}\rho^2_h} + \EE\sbr{\gamma^{2\eta}\rho_\eta^2\vmax^2}\\
& = \sum_{h=1}^{\eta+1}\EE\sbr{\gamma^{2(h-1)}\rho_{h-1}^2\vmax^2} + \sum_{h=1}^\eta \EE\sbr{\gamma^{2(h-1)}\rho_h^2(\vmax^2+1)}\\
& \leq 3 \vmax^2 \sum_{h=1}^\eta \EE\sbr{\gamma^{2(h-1)} \rho_h^2} + \vmax^2\end{align*}
Here in the first line we use the three facts we stated above. In the
second line we collect the terms. In the third line we note that
$\gamma^{2(h-1)}\rho_{h-1}^2 \leq \gamma^{2(h-2)} \rho_{h-1}^2$ since
$\gamma \in (0,1)$, so we can re-index the first summation and group
terms again. 

To simplify further, let $\pmax \defeq \sup_{x,a}
\frac{\pitarget(a \mid x)}{\pilog(a \mid x)}$ denote the largest
importance weight and note that as $\EE_h[p_h] = 1$, we have
\begin{align*}
\sum_{h=1}^\eta \EE\sbr{\gamma^{2(h-1)} \rho_h^2} \leq \sum_{h=1}^\eta \EE\sbr{\gamma^{2(h-1)} \pmax \rho_{h-1}^2\EE_h[w_h]} \leq \ldots \leq \sum_{h=1}^\eta \gamma^{2(h-1)}\pmax^h. \end{align*}
Therefore, our variance bound will be
\begin{align*}
\Var(\hat{V}_\eta^H) \leq 3\vmax^2 \rbr{1 + \sum_{h=1}^\eta \gamma^{2(h-1)}\pmax^h}.
\end{align*}

For the range, we obtain the recursion (for $1 \leq h < \eta$):
\begin{align*}
\abr{\hat{V}_\eta^{H+1-h}} \leq \vmax + \pmax(1+\vmax) + \pmax\gamma \abr{\hat{V}_\eta^{H-h}},
\end{align*}
with the terminal condition $\abr{\vdm^\eta} \leq \vmax$. A somewhat crude upper bound is
\begin{align*}
\abr{\hat{V}_\eta^{H}} \leq 3\vmax\rbr{1 + \sum_{h=1}^\eta \gamma^{h-1}\pmax^h},
\end{align*}
which has a similar form to the variance expression. 

Therefore, Bernstein's inequality reveals that with probability $1-\delta$, we have that the $n$-trajectory empirical averages satisfy
\begin{align}
\abr{\hat{V}_\eta^{H} - \EE\hat{V}_\eta^H} \leq \sqrt{\frac{6 \vmax^2\rbr{1+\sum_{h=1}^\eta \gamma^{2(h-1)}\pmax^h} \log(2/\delta)}{n}} + \frac{6 \vmax\rbr{1+ \sum_{h=1}^\eta \gamma^{h-1}\pmax^h}\log(2/\delta)}{3n}.
\label{eq:rl_dev}
\end{align}
This bound is clearly seen to be montonically increasing in $\eta$,
which is montonically decreasing with $i$ as required. The reason is
that when we increase $\eta$ we add one additional non-negative term
to both the variance and range expressions.

Finally, we must verify that the bound does not decrease too
quickly. For this, we first verify the following elementary fact
\begin{fact}
Let $z \geq 0$ and $t \geq 0$ then
\begin{align*}
\frac{1+\sum_{\tau=1}^tz^{\tau}}{1+\sum_{\tau=1}^{t-1}z^\tau} \leq 1+z.
\end{align*}
\end{fact}
\begin{proof}
Using the geometric series formula, we can rewrite
\begin{align*}
\frac{1 + \sum_{\tau=1}^t z^\tau}{1+\sum_{\tau=1}^{t-1}z^\tau} = 1 + \frac{z^t}{1+\sum_{\tau=1}^{t-1}z^\tau} \leq 1+z.\tag*\qedhere
\end{align*}
\end{proof}
Using the above fact, we can see that the variance bound decreases at
rate $(1+\gamma^2 \pmax)$ and the range bound decreases at rate
$(1+\gamma\pmax)$. The range bound dominates here, since
\begin{align*}
\sqrt{1+\gamma^2\pmax} \leq \sqrt{1+\gamma^2\pmax^2} \leq 1+\gamma\pmax
\end{align*}
Therefore, we may take the decay constant to be $1/(1+\gamma\pmax)$ to
verify~\pref{assum:monotonicity}.
\end{proof}

\section{Details for continuous contextual bandits experiments}
\label{app:ccb_exp}

\subsection{The simulation environment.}

Here, we explain some of the important details of the simulation
environment. The simulator is initialized with a $d_x$ dimensional
context space and action space $[0,1]^d$ for some parameter $d$. For
our experiments we simply take $d=1$. There is also a hidden parameter
matrix $\beta^\star \sim \Ncal(0,I)$ with $\beta^\star \in
\RR^{d\times d_x}$. In each round, contexts are sampled iid from
$\Ncal(0,I)$, then the optimal action $a^\star(x) \defeq
\sigma(\beta^\star x)$, where $\sigma(z) = \tfrac{e^z}{e^z+1}$ is the
standard sigmoid, and the function is applied component-wise. This
optimal action $a^\star$ is used in the design of the reward
functions.

We consider two different reward functions called ``absolute value''
and ``quadratic.'' The first is simply $\ell(a) \defeq 1 - \min(L
\nbr{a- a^\star(x)}_1,1)$, while the latter is $\ell(a) \defeq 1 -
\min(L/4 \sum_{j=1}^d (a_j - a^\star_j(x))^2, 1)$. Here $L$ is the Lipschitz
constant, which is also a configurable.

For policies, the uniform logging policy simply chooses $a \sim
\textrm{Unif}([0,1]^d)$ on each round. Other logging and target
policies are trained via regression on 10
vector-valued regression samples $(x,a^\star(x)+\Ncal(0,0.5\cdot I))$
where $x \sim \Ncal(0,I)$. We use two different regression models:
linear + sigmoid implemented in PyTorch, and a decision 
tree implemented in scikit-learn. Both regression procedures yield
deterministic policies, and in our experiments we take this policies
to be $\pitarget$.

For $\pilog$ we implement two softening techniques
following~\citet{farajtabar2018more}, called ``friendly'' and
``adversarial,'' and both techniques take two parameters
$\alpha,\beta$. Both methods are defined for discrete action spaces,
and to adapt to the continuous setting we partition the continuous
action space into $m$ bins (for one-dimensional spaces). We round the
deterministic action chosen by the regression model to its associated
bin, run the softening procedure to choose a (potentially different)
bin, and then sample an action uniformly from this bin. For higher
dimensional action spaces, we discretize each dimension individually,
so the softening results in a product measure.

Friendly softening with discrete actions is implemented as follows. We
sample $U \sim \textrm{Unif}([-0.5,0.5])$ and then the updated action
is $\pi_{\textrm{det,disc}}(x)$ with probability $\alpha+\beta U$ and
it is uniform over the remaining discrete actions with the remaining
probability. Here $\pi_{\textrm{det,disc}}$ is the deterministic
policy obtained by the regression model, discretized to one of the $m$
bins. Adversarial softening instead is uniform over all discrete
actions with probability $1-(\alpha+\beta U)$ and it is uniform over
all but $\pi_{\textrm{det,disc}}(x)$ with the remaining probability.
In both cases, once we have a discrete action, we sample a continuous
action from the corresponding bin.

The simulator also supports two different kernel functions:
Epanechnikov and boxcar. The boxcar kernel is given by $K(u)
=\frac{1}{2} \one\{\abr{u} \leq 1\}$, while Epanechnikov is $K(u) =
0.75\cdot(1-u^2)\one\{\abr{u} \leq 1\}$. We address boundary bias by
normalizing the kernel appropriately, as opposed to forcing the target
policy to choose actions in the interior.
This issue is also discussed in \citet{kallus2018policy}. 

Finally, we also vary the number of logged samples and the Lipschitz
constant of the loss functions.

\subsection{Reproducibility Checklist}

\indent \textbf{Data collection process.} All data are synthetically generated as described above.

\noindent \textbf{Dataset and Simulation Environment.} We will make the simulation environment publicly available.

\noindent \textbf{Excluded Data.} No excluded data.

\noindent \textbf{Training/Validation/Testing allocation.} 
There is no training/validation/testing setup in off policy
evaluation. Instead all logged data are used for evaluation.

\noindent \textbf{Hyper-parameters.} 
Hyperparameters used in the experimental conditions are: $n \in
10^{1:5}$, $h \in \{2^{-(1:7)}$, $L \in \{0.1, 0.3,1,3,10\}$, in addition to
the other configurable parameters (e.g., softening technique, kernel,
logging policy, target policy). 

\noindent \textbf{Evaluation runs.} 
There are 1000 conditions, each with 30 replicates with
different random seeds.

\noindent \textbf{Description of experiments.} 
For each condition, determined by logging policy, softening technique,
target policy, sample size, lipschitz constant, reward function, and
kernel type, we generate $n$ logged samples following $\pilog$, and
100k samples from $\pitarget$ to estimate the ground truth
$V(\pitarget)$. All fixed-bandwidth estimators and \lepski are
calculated based on the same logged data. The MSE is estimated by
averaging across the 30 replicates, each with different random seed.

For the learning curve in the right panel of~\pref{fig:cb_exp} the
specific condition shown is: uniform logging policy, linear+sigmoid
target policy, $L=3$, absolute value reward, boxcar kernel. MSE
estimates are measured at $n = \{1,3,7\}\times
10^{1:3}\cup\{10,000\}$. We perform 100 replicates for this
experiment.

\noindent \textbf{Measure and Statistics.} 
Results are shown in~\pref{fig:cb_exp}. Statistics are based on
empirical CDF calculated by aggregating the 1000 conditions. Typically
there are no error bars for such plots. Pairwise comparison is based
on paired $t$-test over all pair of methods and conditions, with significance level $0.05$. The learning curve is based on 100
replicates, with error bar corresponding to $\pm 2$ standard errors
shown in the plots.

\noindent \textbf{Computing infrastructure.} 
Experiments were run on Microsoft Azure. 

\section{Details for reinforcement learning experiments}
\label{app:rl_exp}
\subsection{Experiment Details}
\noindent \textbf{Environment Description.} 
We provide brief environment description below. More details can be
found in \citet{thomas2016data, voloshin2019empirical,
  brockman2016openai}.
\begin{packed_itemize}
	\item Mountain car is a classical benchmark from OpenAI Gym. We
      make the same modification as \citet{voloshin2019empirical}. The
      domain has 2-dimensional state space (position and velocity) and
      one-dimensional action $\{\textrm{left,nothing,right}\}$. The
      reward is $r=-1$ for each timestep before reaching the goal. The
      initial state has position uniformly distributed in the discrete set
      $\{-0.4,-0.5,-0.6\}$ with velocity $0$. The horizon is set to be
      $H=250$ and there is an absorbing state at $(0.5,0)$. The domain has
      deterministic dynamics, as well as deterministic, dense reward.
	\item Graph and Graph-POMDP are adopted from
      \citet{voloshin2019empirical}. The Graph domain has horizon 16,
      state space $\{0,1,2,\cdots,31\}$ and action space
      $\{0,1\}$. The initial state is $x_0=0$, and we have the
      state-independent stochastic transition model with
      $\PP(x_{t+1}=2t+1|a=0)=0.75$, $\PP(x_{t+1}=2t+2|a=0)=0.25$,
      $\PP(x_{t+1}=2t+1|a=1)=0.25$, $\PP(x_{t+1}=2t+2|a=1)=0.75$. In
      the dense reward configuration, we have $r(x_t, a_t, x_{t+1}) =
      2(x_{t+1}\mod 2)-1$ $\forall t\leq T$. The sparse reward setting
      has $r(x_t,a_t,x_{t+1})=0$ $\forall t< T-1$ with reward only at
      the last time step, according to the dense reward function. We
      also consider a stochastic reward setting, where we change the
      reward to be $r(x_t, a_t, x_{t+1})\sim \Ncal(2(x_{t+1}\mod 2)-1,
      1)$. Graph-POMDP is a modified version of Graph where states are
      grouped into 6 groups. Only the group information is observed,
      so the states are aliased.
	\item Gridworld is also from \citet{voloshin2019empirical}. The
      state space is an $8\times 8$ grid with four actions [up, down,
        left, right]. The initial state distribution is uniform over
      the left column and top row, while the goal is in the bottom
      right corner. The horizon length is 25. The states belongs to
      four categories: Field, Hole, Goal, Others. The reward at Field
      is -0.005, Hole is -0.5, Goal is 1 and Others is -0.01. The
      exact map can be found in \citet{voloshin2019empirical}.
	\item Hybrid Domain is from \citet{thomas2016data}. It is a
      composition of two other domains from the same study, called
      ModelWin and ModelFail. The ModelFail domain has horizon 2, four
      states $\{s_0, s_1, s_2, s_a\}$ and two actions $\{0,1\}$. The
      agent starts at $s_0$, goes to $s_1$ with reward 1 if $a=0$, and
      goes to $s_2$ with reward if $a=1$. Then it transitions to the
      absorbing state $s_a$. This environment has partial
      observability so that $\{s_0,s_1,s_2\}$ are aliased together.

      In the hybrid domain the absorbing state $s_a$ is replaced with
      a new state $s_1$ in the ModelWin domain.
      This domain has four states $\{s_1,s_2,s_3,s_a\}$. The action
      space is $\{0,1\}$. The agent starts from $s_1$. Upon taking
      action $a=0$, it goes to $s_2$ with probability 0.6 and receives
      reward 1, and goes to $s_3$ with probability 0.4 and reward
      -1. If $a=1$, it does the opposite. From $s_2$ and $s_3$ the
      agent deterministically transitions to $s_1$ with 0 reward.  do
      a deterministic transition back to $s_1$ with 0 reward. The
      horizon here is 20 and $x_{20} = s_a$. The states are fully
      observable.
\end{packed_itemize}

\noindent \paragraph{Models.} 
Instead of experiment with all possible approaches for direct
modeling, which is quite burdensome, we follow the high-level
guidelines provided in Table 3 of \citet{voloshin2019empirical}'s
paper: for Graph, PO-Graph, and Mountain Car we use FQE because these
environments are stochastic and have severe mismatch between logging
and target policy. In contrast, Gridworld has moderate policy
mismatch, so we use $Q^\pi(\lambda)$. For the Hybrid domain, we use a
simple maximum-likelihood approximate model to predict the full
transition operator and rewards, and plan in the model to estimate the
value function.

\paragraph{Policy.} 
For Gridworld and Mountain Car, we use $\epsilon$-Greedy polices as
logging and target policies. To derive these, we first train a base
policy using value iteration and then we take $\pi(a^\star |x) =
1-\epsilon$ and $\pi(a \mid x) = \epsilon/(|\Acal|-1)$ for $a \ne
a^\star(x)$, where $a^\star(x) = \argmax \hat{Q}(x,a)$ for the learned
$\hat{Q}$ function.
In Gridworld, we take the
following policy pairs: $[(1, 0.1),
  (0.6,0.1),(0.2,0.1),(0.1,0.2),(0.1,0.6)]$, where the first argument
is the $\epsilon$ parameter for . For Mountain Car domain, we take the
following policy pairs: $[(0.1, 0), (1,0),(1,0.1),(0.1,1)]$ where the
first argument is the parameter for $\pilog$ and the second is for
$\pitarget$. For the Graph and Graph-POMDP domain, both logging and
target policies are static polices with probability $p$ going left
(marked as $a=0$) and probability $1-p$ going right (marked as $a=1$),
i.e., $\pi(a=0|x) =p$ and $\pi(a=1|x)=1-p$ $\forall x$. In both
environments, we vary $p$ of the logging policy to be $0.2$ and $0.6$,
while setting $p$ for target policy to be $0.8$. For the Hybrid
domain, we use the same policy as \citet{thomas2016data}. For the
first ModelFail part, $\pilog(a=0)=0.88$ and $\pilog(a=1)=0.12$, while
the target policy does the opposite. For the second ModelWin part,
$\pilog(a=0|s_1)=0.73$ and $\pilog(a=1|s_1)=0.27$, and the target
policy does the opposite. For both policies, they select actions
uniformly when $s\in\{s_2,s_3\}$.

\paragraph{Other parameters.} 
For both the Graph and Graph-POMDP, we use $\gamma=0.98$ and $N\in
2^{7:10}$. For Gridworld, $\gamma=0.99$ and $N\in2^{7:9}$. For
Mountain Car, $\gamma=0.96$ and $N\in 2^{8:10}$. For Hybrid,
$\gamma=0.99$ and $N\in\{10,20,50,\cdots, 10000, 20000, 50000\}$. Each
condition is averaged over 100 replicates.

\subsection{Reproducibility Checklist}
\noindent \textbf{Data collection process.} 
All data are synthetically generated as described above.

\noindent \textbf{Dataset and Simulation Environment.}  The Mountain
Car environment is downloadable from OpenAI
\citep{brockman2016openai}. Graph, Graph-POMDP, Gridworld, and the
Hybrid domain are available at
\url{https://github.com/clvoloshin/OPE-tools}, which is the supporting
code for~\citet{voloshin2019empirical}.

\noindent \textbf{Excluded Data.} No excluded data.

\noindent \textbf{Training/Validation/Testing allocation.} 
There is no training/validation/testing setup in off policy
evaluation. Instead all logged data are used for evaluation.

\noindent \textbf{Hyper-parameters.} 
Hyperparameters (apart from those optimized by \lepski) are optimized
followng the guidelines of~\citet{voloshin2019empirical}. For
MountainCar, the direct model is trained using a 2-layer fully connected neural
network with hidden units 64 and 32. The batch size is 32 and convergence is set to be $1e-4$,
network weights are initialized with truncated Normal$(0,0.1)$. For
tabular models, convergence of Graph and Graph-POMDP is $1e-5$ and
Gridworld is $4e-4$.

\noindent \textbf{Evaluation runs.} 
All conditions have 100 replicates with different random seeds.

\noindent \textbf{Description of experiments.} 
For each condition, determined by the choice of environment,
stochastic/deterministic reward, sparse/dense reward,
stochastic/deterministic transition model, logging policy $\pilog$,
target policy $\pitarget$ and sample size $N$. We generate $N$ logged
trajectories following $\pilog$, and 10000 samples from $\pitarget$ to
compute the ground truth $V(\pitarget)$. All baselines and \lepski are
calculated based on the same logged data. The MSE is estimated by
averaging across the 100 replicates, each with different random seed.

\noindent \textbf{Measure and Statistics.} 
Results are shown in~\pref{fig:rl_fig}. Statistics are based on
empirical CDF calculated by aggregating the 106 conditions. Typically
there are no error bars in ECDF plots. Pairwise comparison is based on
paired $t$-test over all pair of methods over all conditions. Each
test has significance level $0.05$. Learning curve is based on Hybrid
domain with 128 replicates, with error bar corresponding to $\pm 2$
standard errors shown in the plots.

\noindent \textbf{Computing infrastructure.} 
RL experiments were conducted in a Linux compute cluster.
 
\bibliography{refs}
\vfill

\end{document}